\documentclass[a4paper]{article}

\usepackage{ifthen}
\newboolean{preprint}
\setboolean{preprint}{false}

\ifthenelse{\boolean{preprint}}{

\newcommand{\titl}{Attack-Aware Noise Calibration for Differential Privacy}
\renewcommand{\paragraph}[1]{\medskip\noindent\textbf{#1.} }
\newcommand{\figwidth}{.45\linewidth}
\title{\titl}
\date{}
\usepackage{authblk}
\author[1,2]{Bogdan Kulynych\footnote{Contributed equally.}}
\newcommand\CoAuthorMark{\footnotemark[\arabic{footnote}]} %
\author[3]{Juan Felipe Gomez\protect\CoAuthorMark}
\author[4]{\authorcr Georgios Kaissis}
\author[3]{Flavio du Pin Calmon}
\author[5]{Carmela Troncoso}

\affil[1]{Lausanne University Hospital (CHUV)}
\affil[2]{University of Lausanne}
\affil[3]{Harvard University}
\affil[4]{Technical University Munich}
\affil[5]{EPFL}

\usepackage[round]{natbib}

\usepackage{libertine}
\usepackage{fullpage}
\usepackage{inconsolata}

}{

\PassOptionsToPackage{round}{natbib}

\newcommand{\figwidth}{.5\linewidth}
\newcommand{\titl}{Attack-Aware Noise Calibration \\ for Differential Privacy}
\usepackage[final]{neurips_2024}
\renewcommand{\paragraph}[1]{\textbf{#1.} }

\author{
    Bogdan Kulynych\thanks{Contributed equally.} \\ Lausanne University Hospital (CHUV) \And
    \newcommand\CoAuthorMark{\footnotemark[\arabic{footnote}]} %
    Juan Felipe Gomez\protect\CoAuthorMark \\ Harvard University \AND
    Georgios Kaissis \\ Technical University Munich \And
    Flavio du Pin Calmon \\ Harvard University \And
    Carmela Troncoso \\ EPFL
}

}

\title{\titl}

\renewcommand{\cite}{\citep}

\usepackage[pdftex,bookmarksnumbered,bookmarksopen,colorlinks,citecolor={blue!70!black},linkcolor={blue!70!black},urlcolor=blue]{hyperref}

\usepackage{microtype}
\usepackage{graphicx}
\usepackage{booktabs} %
\usepackage{multirow}
\usepackage{enumitem}

\usepackage{amsmath}
\usepackage{amssymb}
\usepackage{mathtools}
\usepackage{amsthm}
\usepackage{nicefrac}
\usepackage{bbding}
\usepackage[labelfont=bf]{caption}
\usepackage{subcaption}
\usepackage[group-separator={,},group-minimum-digits={3}]{siunitx}

\usepackage{physics}
\usepackage{tikz}
\usepackage{mathdots}
\usepackage{yhmath}
\usepackage{cancel}
\usepackage{color}
\usepackage{siunitx}
\usepackage{array}
\usepackage{multirow}
\usepackage{gensymb}
\usepackage{tabularx}
\usepackage{extarrows}
\usepackage{wrapfig}
\usepackage[capitalize,noabbrev]{cleveref}
\crefname{equation}{Eq.}{Eqs.}

\usepackage{algorithm}
\usepackage[noend]{algpseudocode}

\usepackage{pgfplots}
\usetikzlibrary{fadings}
\usetikzlibrary{patterns}
\usetikzlibrary{shadows.blur}
\usetikzlibrary{shapes}

\DeclareMathOperator*{\E}{\mathbb{E}}
\DeclareMathOperator*{\argmin}{argmin}

\newcommand{\define}{~\smash{\triangleq}~}
\newcommand{\newterm}{\textit}

\newcommand{\sX}{\mathbb{X}}

\newcommand{\sD}{\mathbb{D}}
\newcommand{\sR}{\mathbb{R}}

\newcommand{\train}{M}
\newcommand{\dataset}{S}

\newcommand{\target}{\smash{^\star}}
\newcommand{\tradeoff}{T}

\newcommand{\fpr}{\alpha}
\newcommand{\fnr}{\beta}
\newcommand{\alg}{q}
\newcommand{\mechanism}{M}

\newcommand{\noiserv}{Z}

\newcommand{\noiseparams}{\omega}
\newcommand{\noiseparam}{\omega}
\newcommand{\noiseparamspace}{\Omega}

\newcommand{\Th}{\Theta}

\theoremstyle{plain}
\usepackage{thmtools} 
\usepackage{thm-restate}
\declaretheorem[name=Proposition,numberwithin=section]{proposition}
\declaretheorem[name=Theorem,numberwithin=section]{theorem}

\newtheorem{corollary}[theorem]{Corollary}
\theoremstyle{definition}
\newtheorem{definition}[theorem]{Definition}

\theoremstyle{remark}

\Crefname{definition}{Def.}{Defs.}

\definecolor{seabornblue}{HTML}{1f77b4}
\definecolor{seabornorange}{HTML}{ff7f0e}
\definecolor{seaborngreen}{HTML}{2ca02c}
\definecolor{seabornred}{HTML}{d62728}

\begin{document}

\maketitle

\begin{abstract}
Differential privacy (DP) is a widely used approach for mitigating privacy risks when training machine learning models on sensitive data. DP mechanisms add noise during training to limit the risk of information leakage. The scale of the added noise is critical, as it determines the trade-off between privacy and utility. 
The standard practice is to select the noise scale to satisfy a given \emph{privacy budget} $\varepsilon$. This privacy budget is in turn interpreted in terms of operational \emph{attack risks}, such as accuracy, sensitivity, and specificity of inference attacks aimed to recover information about the training data records. 
We show that first calibrating the noise scale to a privacy budget $\varepsilon$, and then translating $\varepsilon$ to attack risk leads to overly conservative risk assessments and unnecessarily low utility.
Instead, we propose methods to directly calibrate the noise scale to a desired attack risk level, bypassing the step of choosing $\varepsilon$. For a given notion of attack risk, our approach significantly decreases noise scale, leading to increased utility at the same level of privacy. We empirically demonstrate that calibrating noise to attack sensitivity/specificity, rather than $\varepsilon$, when training privacy-preserving ML models substantially improves model accuracy for the same risk level. Our work provides a principled and practical way to improve the utility of privacy-preserving ML without compromising on privacy.
\end{abstract}

\begin{figure*}[t]
    \centering
    \hspace{2.5em}
    \begin{subfigure}[b]{.45\textwidth}
        \centering
        \sffamily
        \textbf{\textcolor{seabornblue}{Standard Calibration}}\\[.5em]
        \resizebox{\textwidth}{!}{
        \begin{tabular}{ccccc}
            \text{DP Parameters} & $\rightarrow$ & \text{Noise Scale} & $\rightarrow$ & 
            \text{Attack Risk} \\
            \textcolor{gray}{\footnotesize $(\varepsilon, \delta)$} & & \textcolor{gray}{\footnotesize $\sigma$}
            & & 
            \textcolor{gray}{\footnotesize $1 - \fnr$ (TPR), $\fpr$ (FPR)} \\
        \end{tabular}
        }
        \centering
    \end{subfigure}
    \begin{subfigure}[b]{.45\textwidth}
        \centering
        \sffamily
        \textbf{\textcolor{seabornorange}{Our Method}}\\[.5em]
        \resizebox{0.65\textwidth}{!}{
        \begin{tabular}{ccc}
            \text{Attack Risk} & $\rightarrow$ & \text{Noise Scale} \\
            \textcolor{gray}{\footnotesize $1 - \fnr$ (TPR), $\fpr$ (FPR)} & & \textcolor{gray}{\footnotesize $\sigma$}
        \end{tabular}
        }
    \end{subfigure}\\
    \vspace{1.5em}
    \resizebox{.95\textwidth}{!}{\sffamily \textcolor{seabornorange}{\textbf{Direct calibration of noise to attack risk}} increases utility compared to \textcolor{seabornblue}{\textbf{the standard calibration}} at the same level of risk:}\\[.25em]
    \resizebox{.30\textwidth}{!}{\sffamily GPT-2 on SST-2 (text sentiment classification)}\\[.1em]
    \resizebox{\textwidth}{!}{\input{images/gpt2_err_rates_calibration.pgf}}\\[.1em]
    \resizebox{.25\textwidth}{!}{\sffamily CNN on CIFAR-10 (image classification)}
    \resizebox{\textwidth}{!}{\input{images/cifar10_err_rates_calibration.pgf}}
    \caption{Test accuracy (x-axis) of a privately finetuned GPT-2 on SST-2 text sentiment classification dataset (top) and a convolutional neural network on CIFAR-10 image classification dataset (bottom). The DP noise is calibrated to guarantee at most a certain level of privacy attack sensitivity (y-axis) at three possible attack false-positive rates $\fpr \in \{0.01, 0.05, 0.1\}$. See \cref{sec:experiments} for details.
    }
    \label{fig:cifar10-err-rates-calibration}
    \vspace{-1em}

\end{figure*}

\section{Introduction}
\label{sec:intro}
Machine learning and statistical models can leak information about individuals in their training data, which can be recovered by membership inference, attribute inference, and reconstruction attacks~\cite{fredrikson2015model,shokri2017membership,yeom2018privacy,balle2022reconstructing}.
The most common defenses against these attacks are based on differential privacy (DP)~\cite{dwork2014algorithmic}. Differential privacy introduces noise to either the data, the training algorithm, or the model parameters~\cite{chaudhuri2011differentially}. This noise provably limits the adversary's ability to run successful  attacks at the cost of reducing the utility of the model.

In DP, the parameters $\varepsilon$ and $\delta$ control the privacy-utility trade-off. These parameters determine the scale (e.g., variance) of the noise added during training: Smaller values of these parameters correspond to larger noise. Larger noise provides stronger privacy guarantees but reduces the utility of the trained model. Typically, $\delta$ is set to a small fixed value (usually between $10^{-8}$ and $10^{-5}$), leaving $\varepsilon$ as the primary tunable parameter. 
Without additional analyses, the values of parameters $(\varepsilon, \delta)$ alone do not provide a tangible and intuitive operational notion of privacy risk~\cite{nanayakkara2023chances}. This begs the question: how should practitioners, regulators, and data subjects decide on acceptable values of $\varepsilon$ and $\delta$ and calibrate the noise scale to achieve a desired level of protection?

A standard way of assigning  operational meaning to DP parameters is mapping them to \emph{attack risks}. One common approach is computing attacker's posterior belief (or equivalently, accuracy or advantage) of membership inference attacks, that concrete values of $(\varepsilon, \delta)$ allow~\cite{wood2018differential}. An alternative is to compute the trade-off between sensitivity and specificity of feasible membership inference attacks~\cite{wasserman2010statistical,kairouz2015composition,dong2019gaussian}, which was recently shown to also be directly related to success of record reconstruction attacks~\cite{hayes2024bounding, kaissis2023bounding}. Such approaches map $(\varepsilon, \delta)$ to a quantifiable level of risk for individuals whose data is present in the dataset.
Studies have shown that such risk-based measures are the most useful way to interpret the guarantees afforded by DP for practitioners and data subjects~\cite{cummings2021need, franzen2022private, nanayakkara2023chances}.

In this work, we show that directly calibrating the level of noise to satisfy a given level of attack risk, as opposed to satisfying a certain $\varepsilon$, enables a significant increase in utility (see \cref{fig:cifar10-err-rates-calibration}). We enable this direct calibration to attack risk by working under $f$-DP \cite{dong2019gaussian}, a hypothesis testing interpretation of DP. In particular, we extend the tight privacy analysis method by \citet{doroshenko2022connect} to directly estimate operational privacy risk notions in $f$-DP. Then, we
use our extended algorithm to directly calibrate the level of noise to satisfy a given level of attack risk.
Concretely, our contributions are:
\renewcommand{\thefootnote}{\fnsymbol{footnote}}
\begin{enumerate}[itemsep=0em]
    \vspace{-.5em}
    \item We provide efficient methods for calibrating noise to (a) maximum accuracy (equivalently, advantage), (b) sensitivity and specificity of membership inference attacks, in any DP mechanism, including DP-SGD~\cite{abadi2016deep} with arbitrarily many steps.
    \item We empirically show that our calibration methods reduce the required noise scale for a given level of privacy risk, up to 2$\times$ as compared to standard methods for choosing DP parameters. In a private language modeling task with GPT-2~\cite{radford2019language}, we demonstrate that the decrease in noise can translate to a 18 p.p. gain in classification accuracy. 
    \item We demonstrate that relying on membership inference accuracy as an interpretation of privacy risk, as is common practice, can increase attack power in privacy-critical regimes, and that calibration for sensitivity and specificity does not suffer from this drawback.
    \item We provide a Python package which implements our algorithms for analyzing DP mechanisms in terms of the interpretable $f$-DP guarantees, and calibrating to operational risks:
    \begin{center}
    \href{https://github.com/Felipe-Gomez/riskcal}{github.com/Felipe-Gomez/riskcal}
    \end{center}
    \vspace{-.5em}
\end{enumerate}
Ultimately, we advocate for practitioners to calibrate the noise level in privacy-preserving machine learning algorithms to a sensitivity and specificity constraint under $f$-DP as outlined in \cref{sec:err-rates-calibration}.

\paragraph{Related Work} Prior work has studied methods for communicating the privacy guarantees afforded by differential privacy~\cite{nanayakkara2023chances,nanayakkara2022visualizing,franzen2022private,mehner2021towards,wood2018differential}, and introduced various principled methods for choosing the privacy parameters ~\cite{abowd2015revisiting,nissim2014redrawing,hsu2014differential}. Unlike our approach, these works assume that the mechanisms are calibrated to a given $\varepsilon$ privacy budget parameter, and do not aim to directly set the privacy guarantees in terms of operational notions of privacy risk.

\citet{cherubin2024closed, ghazi2023total, izzo2022provable, mahloujifar2022optimal} use variants of DP that directly limit the advantage of membership inference attacks. We show that calibrating noise to a given level of advantage can increase privacy risk in security-critical regimes and provide methods that mitigate this issue. \citet{leemann2024gaussian} provide methods for evaluating the success of membership inference attacks under a weaker threat model than in DP. Unlike their work, we preserve the standard strong threat model in differential privacy but set and report the privacy guarantees in terms of an operational notion of risk under $f$-DP as opposed to the $\varepsilon$ parameter.

\section{Problem Statement}

\subsection{Preliminaries}
\paragraph{Setup and notation} Let $\sD^n$ denote the set of all datasets of size $n$ over a space $\sD$, and let $\dataset \simeq \dataset'$ denote a neighboring relation, e.g. $\dataset, \dataset'$ that differ by one datapoint. We study randomized algorithms (\emph{mechanisms}) $\mechanism(S)$ that take as input a dataset $\dataset \in 2^{\sD}$, and output the result of a computation, e.g., statistical queries or an ML model. We denote the output domain of the mechanism by $\Th$. For ease of presentation, we mainly consider randomized mechanisms that are parameterized by a single noise parameter $\noiseparam \in \noiseparamspace$, but our results extend to mechanisms with multiple parameters. For example, in the \newterm{Gaussian mechanism}~\cite{dwork2014algorithmic}, $\mechanism(S) = \alg(S) + \noiserv$, where $\noiserv \sim \mathcal{N}(0, \sigma^2)$ and $\alg(S)$ is a non-private statistical algorithm, the parameter is $\noiseparam = \sigma$ with $\noiseparamspace = \sR^{\geq 0}$. We denote a parameterized mechanism by $\mechanism_\noiseparams(S)$. We summarize the notation in \cref{tab:notation} in the Appendix.

\paragraph{Differential Privacy}
\label{sec:background-dp}
For any $\gamma \geq 0$, we define the hockey-stick divergence from distribution $P$ to $Q$ over a domain $\mathcal{O}$ by
\begin{equation}\label{eq:hockeystick-def}
    D_{\gamma}(P ~\|~ Q) \define \sup_{E}  Q(E) - \gamma P(E)
\end{equation}
where the supremum is taken over all measurable sets $E \subseteq \mathcal{O}$. We define differential privacy (DP)~\cite{dwork2006calibrating} as follows:

\begin{definition}\label{def:dp}
    A mechanism $\mechanism(\cdot)$ satisfies $(\varepsilon, \delta)$-DP iff  $\sup_{S \simeq S'} D_{e^\varepsilon}(\mechanism(\dataset)~\|~ \mechanism(\dataset')) \leq \delta$.
\end{definition}
Lower values of $\varepsilon$ and $\delta$ mean more privacy which in turn requires more noise, and vice versa. In the rest of the paper we assume that a larger value of the parameter $\noiseparams \in \noiseparamspace$ for $\noiseparamspace \subseteq \sR$, e.g., standard deviation of Gaussian noise $\noiseparams = \sigma$ in the Gaussian mechanism, means that the mechanism $\mechanism_\noiseparams(\cdot)$ is more noisy, which translates into a higher level of privacy (smaller $\varepsilon, \delta$), but lower utility.

Most DP algorithms satisfy a collection of $(\varepsilon, \delta)$-DP guarantees. We define the \newterm{privacy profile}~\cite{balle2018improving}, or \newterm{privacy curve}~\cite{gopi2021numerical,alghamdi2022saddle} of a mechanism as:
\begin{definition}\label{def:profile}
    A parameterized mechanism $\mechanism_\noiseparams(\cdot)$ has a privacy profile $\varepsilon_\noiseparams: [0,1] \rightarrow \sR$ if for every $\delta \in [0, 1]$, $\mechanism_\noiseparams(\cdot)$ is $(\varepsilon(\delta), \delta)$-DP. 
\end{definition}
We refer to the function $\delta_\noiseparams(\varepsilon)$, defined analogously, also as the privacy profile. 

\paragraph{DP-SGD} A common algorithm for training neural networks with DP guarantees is DP-SGD~\cite{abadi2016deep}. The basic building block of DP-SGD is the \newterm{subsampled Gaussian mechanism}, defined as $\mechanism(S) = \alg(\smash{\mathsf{PoissonSample}_p} \circ S) + Z$, where $Z \sim \mathcal{N}(0, \Delta_2^2 \cdot \sigma^2 \cdot I_d)$, and $\smash{\mathsf{PoissonSample}_p}$ is a procedure which subsamples a dataset $S$ such that every record has the same probability $p \in (0, 1)$ to be in the subsample. DP-SGD, parameterized by $p, \sigma,$ and $T \geq 1$, is a repeated application of the subsampled Gaussian mechanism: $\mechanism^{(1)} \circ \mechanism^{(2)} \circ \cdots \circ \mechanism^{(T)}(S)$, where $\alg^{(i)}(\cdot)$ is a single step of gradient descent with per-record gradient clipping to $\Delta_2$ Euclidean norm. In line with a standard practice~\cite{dpfy}, we regard all parameters but $\sigma$ as fixed, thus $\noiseparams = \sigma$.

Privacy profiles for mechanisms such as DP-SGD are computed via numerical algorithms called \newterm{accountants}~\cite[see, e.g., ][]{abadi2016deep,gopi2021numerical,doroshenko2022connect,alghamdi2022saddle}. 
These algorithms compute the achievable privacy profile to accuracy nearly matching the lower bound of a privacy audit where the adversary is free to choose the entire (pathological or realistic) training dataset \cite{nasr2021adversary, nasr2023tight}. Given these results, we regard the analyses of these accountants as tight, and use them for calibration to a particular $(\varepsilon,\delta)$-DP constraint. 

\paragraph{Standard Calibration} The procedure of choosing the parameter $\noiseparam \in \noiseparamspace$ to satisfy a given level of privacy is called \newterm{calibration}. In \newterm{standard calibration}, one chooses $\noiseparam$ given a target DP guarantee $\varepsilon\target$ and an accountant that supplies a privacy profile $\varepsilon_\noiseparams(\delta)$ for any noise parameter $\noiseparams \in \noiseparamspace$, to ensure that $M_\noiseparams(S)$ satisfies $(\varepsilon\target, \delta\target)$-DP:
\begin{align}
    \vspace{-1em}
    \label{eq:standard-calibration}
    \min_{\noiseparams \in \noiseparamspace} \: \noiseparams \quad\text{ s.t. } \varepsilon_{\noiseparams}(\delta\target) \geq \varepsilon\target,
\end{align}
with $\delta\target$ set by convention to $\delta\target = \nicefrac{1}{c \cdot n}$, where $n$ is the dataset size, and $c>1$~\cite[see, e.g.,][]{dpfy,near2023guidelines}. 
The parameter $\varepsilon\target$ is also commonly chosen by convention between $2$ and $10$ for privacy-persevering ML algorithms with practical utility~\cite{dpfy}.
In \cref{eq:standard-calibration} and the rest the paper we denote by $\target$ the target value of privacy risk. 

After calibration, the $(\varepsilon, \delta)$ parameters are often mapped to some operational notation of privacy attack risk for interpretability. In the next section, we introduce the hypothesis testing framework of DP,  $f$-DP, and the notions of risk that $(\varepsilon, \delta)$ parameters are often mapped to. In contrast to standard calibration, in \cref{sec:intro-to-calibration}, we calibrate $\noiseparams$ to directly minimize these privacy risks.

\subsection{Operational Privacy Risks}
\label{sec:operational-meanings}

We can interpret differential privacy through the lens of membership inference attacks (MIAs) in the so-called \newterm{strong-adversary model}~\cite[see, e.g., ][]{nasr2021adversary}. In this framework, the adversary aims to determine whether a given output $\theta \in \Theta$ came from $\mechanism(S)$ or $\mechanism(S')$, where $S' = S \cup \{z\}$ for some target example $z \in \sD$.\footnote{We use add relation in this exposition, i.e., $S \simeq S'$ iff $S' = S \cup \{z\}$, but our results hold for any relation.} The adversary has access to the mechanism $\mechanism(\cdot)$, the dataset $S$, and the target example $z \in \sD$. Such an attack is equivalent to a binary hypothesis test~\cite{wasserman2010statistical, kairouz2015composition, dong2019gaussian}:
\begin{align}
H_0: \theta \sim \mechanism(S), \quad H_1: \theta \sim \mechanism(S'),
\end{align}
where the MIA is modelled as a test $\phi: \Theta \rightarrow [0,1]$ that maps a given mechanism output $\theta$ to the probability of the null hypothesis $H_0$ being rejected. We can analyze this hypothesis test through the trade-off between the achievable \newterm{false positive rate} (FPR) $\fpr_\phi \define \smash{\E_{M(S)}}[\phi]$ and \newterm{false negative rate} (FNR) $\fnr_\phi \define 1 - \smash{\E_{M(S')}}[\phi]$, where the expectations are taken over the coin flips in the mechanism.\footnote{Note that sensitivity (TPR) is $1 - \beta$ and specificity (TNR) is $1 - \alpha$.} \citet{dong2019gaussian} formalize the \newterm{trade-off function} and define \newterm{$f$-DP} as follows:
\begin{definition}\label{def:trade-off}
    A \newterm{trade-off function} $\tradeoff(\mechanism(S), \mechanism(S')): [0, 1] \rightarrow [0, 1]$ outputs the FNR of the most powerful attack at any given level $\alpha \in [0,1]$:
\begin{equation}
    \tradeoff(\mechanism(S), \mechanism(S'))(\fpr) = \inf_{\phi:~\Theta \rightarrow [0, 1]} \{ \fnr_{\phi} \mid \fpr_\phi \leq \fpr\}
\end{equation}
See \cref{fig:dp-curve} in the Appendix for an illustration.
\end{definition}
\begin{definition}\label{def:fdp}\label{eq:def_of_fdp}
A mechanism $\mechanism(\cdot)$ satisfies $f$-DP, where $f$ is the trade-off curve for some other mechanism, if for all $\alpha \in [0,1]$, we have
$\inf_{S \simeq S'} \tradeoff(\mechanism(S), \mechanism(S'))(\alpha) \geq f(\alpha)$.
\end{definition}
Next, we state the equivalence between $(\varepsilon, \delta)$-DP guarantees and $f$-DP guarantees.
\begin{proposition}[\citet{dong2019gaussian}]\label{stmt:dp-to-f}\label{stmt:profiles-to-f}
    If a mechanism $\mechanism(\cdot)$ is $(\varepsilon, \delta)$-DP, then it is $f$-DP with 
    \begin{equation}
        \label{eq:dp-to-f}
        f(\fpr) = \max \{ 0, \ 1 - \delta - e^{\varepsilon} \fpr, \  e^{-\varepsilon} \cdot (1 - \delta - \fpr)\}. 
    \end{equation}
    Moreover, a mechanism $\mechanism(\cdot)$ satisfies $(\varepsilon(\delta), \delta)$-DP for all $\delta \in [0, 1]$ iff it is $f$-DP with 
    \begin{equation}
        \label{eq:profiles-to-f}
        f(\fpr) = \sup_{\delta \in [0, 1]} \max \{ 0, \ 1 - \delta - e^{\varepsilon(\delta)} \fpr, \  e^{-\varepsilon(\delta)} \cdot (1 - \delta - \fpr)\}. 
    \end{equation}
    \vspace{-.8em}
\end{proposition}
We overview three particular notions of attack risk: advantage/accuracy of MIAs, FPR/FNR of MIAs, and reconstruction robustness. These risks can be thought of as summary statistics of the $f$ curve. 

\paragraph{Advantage/Accuracy}
\citet{wood2018differential} proposed\footnote{\citet{wood2018differential} used \emph{posterior belief}, which is equivalent to accuracy under uniform prior.} to measure the attack risk as the maximum achievable attack accuracy. To avoid confusion with task accuracy, we use \newterm{advantage} over random guessing, which is the difference between the attack TPR $1 - \beta_\phi$ and FNR $\alpha_\phi$:
\begin{equation}\label{eq:adv}
    \eta \define \sup_{S \simeq S'} \sup_{\phi:~\Theta \rightarrow [0, 1]} 1 - \beta_\phi - \alpha_\phi.
\end{equation}
The advantage $\eta$ is a linear transformation of the maximum attack accuracy $\sup \nicefrac{1}{2} \cdot (1 - \beta_\phi) + \nicefrac{1}{2} \cdot (1 - \alpha_\phi)$, where supremum is over $S \simeq S'$ and $\phi: \Theta \rightarrow [0, 1]$. Moreover, $\eta$ can be obtained from a fixed point $\alpha^* = f(\alpha^*)$ of the $f$ curve as $1 - 2\alpha^*$, and it is bounded given an $(\varepsilon, \delta)$-DP guarantee:
\begin{proposition}[\citet{kairouz2015composition}]\label{stmt:dp-to-adv} %
If a mechanism $\mechanism(\cdot)$ is $(\varepsilon, \delta)$-DP, then we have:
\begin{align}
    \eta \leq \frac{e^\varepsilon -1  + 2\delta}{e^\varepsilon + 1}.
    \label{eq:dp-to-adv}
\end{align}
\end{proposition}
\paragraph{FPR/FNR Risk}
Recent work \cite{carlini2022membership,rezaei2021difficulty} has argued that MIAs are a relevant threat only when the attack true positive rate $1 - \fnr_\phi$ is high at low enough $\fpr_\phi$. As a concrete notion of risk, we thus consider minimum level of attack FNR $\fnr\target$ within an FPR region $\fpr \in [0, \fpr\target]$, where $\fpr\target$ is a low value. This approach is similar to the statistically significant p-values often used in the sciences. Following the scientific standards and \citet{carlini2022membership}, we consider $\fpr\target \in \{0.01, 0.05, 0.1\}$. 

\paragraph{Reconstruction Robustness}
Another privacy threat is the reconstruction of training data records~\cite[see, e.g., ][]{balle2022reconstructing}. Denoting by $R(\theta; z)$ an attack that aims to reconstruct $z$, its success probability can be formalized as $\rho \define \Pr[\ell(z, R(\theta; z)) \leq \gamma]$ over $\theta \sim \mechanism(S \cup \{ z \}), z \sim \pi$ for some loss function $\ell: \sD^2 \rightarrow \sR$ and prior $\pi$. \citet{kaissis2023bounding} showed that MIA error rates bound reconstruction success as $\rho \leq 1 - f(\kappa_\gamma)$ for an appropriate choice of $\kappa_\gamma$. Therefore, the FPR/FNR trade-off curve can also be thought as a notion of robustness to reconstruction attacks.

\subsection{Our Objective: Attack-Aware Noise Calibration}
\label{sec:intro-to-calibration}
The standard practice in DP is to calibrate the noise scale $\noiseparams$ of a mechanism $\mechanism_\noiseparams(\cdot)$ to some target $(\varepsilon\target,\delta \target)$-DP guarantee, with $\varepsilon\target$ from a recommended range, e.g., $\varepsilon\target \in [2, 10]$, and $\delta\target$ fixed to $\delta\target < \nicefrac{1}{n}$, as in \cref{eq:standard-calibration}. Then, the privacy guarantees provided by the chosen $(\varepsilon\target, \delta\target)$ are obtained by mapping these values to bounds on sensitivity and specificity (by \cref{stmt:dp-to-f}) or advantage (by \cref{stmt:dp-to-adv}) of membership inference attacks. In this work, we show that if the goal is to provide an operational and interpretable guarantee such as attack advantage or FPR/FNR, this approach leads to unnecessarily pessimistic noise requirements and a deterioration in utility due to the intermediate step of setting $(\varepsilon\target, \delta\target)$. We show it is possible to skip this intermediate step by using the hypothesis-testing interpretation of DP to \emph{directly} calibrate noise to operational notions of privacy risk. In practice, this means replacing the constraint in \cref{eq:standard-calibration} with an operational notion of risk:
\begin{align}
    \label{eq:generic-calibration}
    \min_{\noiseparams \in \noiseparamspace} \: \noiseparams \quad\text{ s.t. } \texttt{risk}_\noiseparams \leq \texttt{threshold}\target.
\end{align}
Solving this optimization problem requires two components. First, a way to optimize $\noiseparams$ given a method to compute $\texttt{risk}_\noiseparams$. As we assume that risk is monotonic in $\noiseparams$, \cref{eq:generic-calibration} can be solved
via binary search~\cite[see, e.g.,][]{pytorch} using calls to the $\texttt{risk}_\noiseparams$ function to an arbitrary precision. Second, we need a way to compute $\texttt{risk}_\noiseparams$ for any value $\noiseparams$. In the next section, we provide efficient methods for doing so for general DP mechanisms, including composed mechanisms such as DP-SGD, by extending the tight privacy analysis from \citet{doroshenko2022connect} to computing $f$-DP. Having these methods, we instantiate \cref{eq:generic-calibration} for the notions of risks introduced in \cref{sec:operational-meanings}.

\section{Numeric Calibration to Attack Risks}
\label{sec:method}
In this section, we provide methods for calibrating DP mechanisms to the notions of privacy risk in \cref{sec:operational-meanings}. As a first step, we introduce the core technical building blocks of our calibration method: methods for evaluating advantage $\eta_\noiseparams$ and the trade-off curve $f_\noiseparams(\alpha)$ for a given value of $\noiseparams$.

\paragraph{Dominating Pairs and PLRVs}
We make use of two concepts, originally developed in the context of computing tight privacy profiles under composition: \newterm{dominating pairs}~\cite{characteristic} and \newterm{privacy loss random variables} (PLRV)~\cite{dwork2016concentrated}.
\begin{definition}\label{def:dom-pair}
    We say that a pair of distributions $(P,Q)$ is a \newterm{dominating pair} for a mechanism $\mechanism(\cdot)$ if for every $\varepsilon \in \mathbb{R}$, we have $\sup_{S \simeq S'} D_{e^\varepsilon}(\mechanism(S) ~\|~ \mechanism(S')) \leq D_{e^\varepsilon}(P ~\|~ Q)$.
\end{definition}
Importantly, a dominating pair also provides a lower bound on the trade-off curve of a mechanism: 
\begin{restatable}{proposition}{pairstof}
\label{stmt:pairs-to-f}
    If $(P,Q)$ is a dominating pair for a mechanism $\mechanism$, then for $\fpr \in [0,1]$,
    \begin{align}
        \inf_{S \simeq S'} T(M(S), M(S'))(\fpr) \geq T(P,Q)(\fpr).
    \end{align}
\end{restatable}
The proofs of this and all the following statements are in \cref{app:proofs}. 
\cref{stmt:pairs-to-f} implies that a mechanism $M(\cdot)$ is $f$-DP with $f = T(P,Q)$. Next, we introduce privacy loss random variables, which provide a natural parameterization of the curve $T(P,Q)$.
\begin{definition}\label{def:plrv}
    Suppose that a mechanism $\mechanism(\cdot)$ has a discrete-valued dominating pair $(P, Q)$. Then, we define the \emph{privacy loss random variables} (PLRVs) $(X, Y)$ as $Y \define \log \nicefrac{Q(o)}{P(o)}$, with $o \sim Q$, and $X \define \log \nicefrac{Q(o')}{P(o')}$ with $o' \sim P$.
\end{definition} 
We can now state the result which serves as a main building block for our calibration algorithms, and forms the main theoretical contribution of our work.
\begin{restatable}[Accounting for advantage and $f$-DP with PLRVs]{theorem}{plrvtorisks}
    \label{stmt:plrv-to-risks}
    Suppose that a mechanism $\mechanism(\cdot)$ has a discrete-valued dominating pair $(P, Q)$ with associated PLRVs $(X,Y)$. The attack advantage $\eta$ for this mechanism is bounded:
    \begin{equation}\label{eq:plrv-to-adv}
    \eta \leq \Pr[Y > 0] - \Pr[X > 0].
    \end{equation}
    Moreover, for any $\tau \in \mathbb{R} \cup \{\infty, -\infty\}$ and $\gamma \in [0,1]$, define
    \begin{equation}
        \beta^*(\tau, \gamma) = \Pr[Y \leq \tau] - \gamma \Pr[Y = \tau].
        \label{eq:plrv-to-f}
    \end{equation}
    For any level $\fpr \in [0,1]$, choosing $\tau = (1-\alpha)$-quantile of $X$ and $\gamma = 
    \frac{\alpha - \Pr[X > \tau]}{\Pr[X=\tau]}$ guarantees that $T(P,Q)(\fpr) = \beta^*(\tau, \gamma)$.
\end{restatable}
To show this, we use the Neyman-Pearson lemma to explicitly parameterize the most powerful attack at level $\alpha$ in terms the threshold $\tau$ on the Neyman-Pearson test statistic and the probability $\gamma$ of guessing when the test statistic exactly equals the threshold. See \cref{app:proof-plrv-to-risks} for the detailed proof.

We remark that similar results for the trade-off curve appear in \cite{characteristic} without the $\gamma$ terms, as \citeauthor{characteristic} assume continuous PLRVs $(X,Y)$. In our work, we rely on the technique due to \citet{doroshenko2022connect}, summarized in \cref{app:connect-the-dots}, which \emph{discretizes} continuous mechanisms such as the subsampled Gaussian in DP-SGD, and provides a dominating pair that is \emph{discrete} and finitely supported over an evenly spaced grid. As the dominating pairs are discrete, the $\gamma$ terms are non-zero, thus are necessary to fully reconstruct the trade-off curve. 

\subsection{Calibration to Advantage}
\label{sec:adv-calibration}
First, we show how to instantiate \cref{eq:generic-calibration} to calibrate noise to a target advantage $\eta\target \in [0, 1]$. Let $\eta_\noiseparams$ denote the advantage of the mechanism $M_\noiseparams(\cdot)$ as defined in \cref{eq:adv}:
\begin{align}\label{eq:adv-calibration}
    \min_{\noiseparams \in \noiseparamspace} \noiseparams \quad \text{ s.t. } \quad \eta_\noiseparams \leq \eta\target.
\end{align}
Given the PLRVs  $(X_\noiseparams, Y_\noiseparams)$, we can obtain a substantially tighter bound than converting $(\varepsilon, \delta)$ guarantees using \cref{stmt:dp-to-adv} under standard calibration. Specifically, \cref{stmt:plrv-to-risks} provides the following way to solve the problem:
\begin{align}\label{eq:adv-calibration-plrv}
    \min_{\noiseparams \in \noiseparamspace} \noiseparams \quad \text{ s.t. } \quad \Pr[Y_\noiseparams > 0] - \Pr[X_\noiseparams > 0] \leq \eta\target
\end{align}
We call this approach \newterm{advantage calibration}, and show how to practically implement it in \cref{alg:get-adv,alg:direct-adv-calibration} in the Appendix. Given a method for obtaining valid PLRVs $X_\noiseparams, Y_\noiseparams$ for any $\noiseparams$, such as the one by \citet{doroshenko2022connect}, advantage calibration is \emph{guaranteed} to ensure bounded advantage, which follows by combining \cref{stmt:pairs-to-f,stmt:plrv-to-risks}:
\begin{restatable}{proposition}{advcorrectness}
    Given PLRVs $(X_\noiseparams, Y_\noiseparams)$ of a discrete-valued dominating pair of a mechanism $\mechanism_\noiseparams(\cdot)$, choosing $\noiseparams^*$ using \cref{eq:adv-calibration-plrv} ensures $\eta_{\noiseparams^*} \leq \eta\target$.
\end{restatable}

\paragraph{Utility Benefits}
\label{sec:quick-demo-adv-calibration}
We demonstrate how calibration for a given level of attack advantage can increase utility. As a mechanism to calibrate, we  consider DP-SGD with $p = 0.001$ subsampling rate, $T = \num{10000}$ iterations, and assume that $\delta\target = 10^{-5}$. Our goal is to compare the noise scale $\sigma$ obtained via advantage calibration to the standard approach.

As a baseline, we choose $\sigma$ using standard calibration in \cref{eq:standard-calibration}, and convert the resulting $(\varepsilon, \delta)$ guarantees to advantage using \cref{stmt:dp-to-adv}. We detail this procedure in \cref{alg:standard-adv-calibration} in the Appendix.
We consider target values of advantage $\eta\target \in [0.01, 0.25]$. As we show in \cref{fig:quick-demo-adv-calibration}, our direct calibration procedure enables to reduce the noise scale by up to $3.5\times$.

\paragraph{Pitfalls of Calibrating for Advantage}
\label{sec:trade-off}
Calibration to a given level of membership advantage is a compelling idea due to the decrease in noise required to achieve better utility at the same level of risk as with the standard approach. Despite this increase in utility, we caution that this approach comes with a deterioration of privacy guarantees other than maximum advantage compared to standard calibration.
Concretely, it allows for \emph{increased attack TPR} in the privacy-critical regime of low attack FPR (see \cref{sec:operational-meanings}). The next result quantifies this pitfall: 
\begin{restatable}[Cost of advantage calibration]{proposition}{advpitfalls}
    \label{stmt:adv-pitfalls}
    Fix a dataset size $n > 1$, and a target level of attack advantage $\eta\target \in (\delta\target, 1)$, where $\delta\target = \nicefrac{1}{c \cdot n}$ for some $c > 1$. For any $0 < \fpr < \frac{1 - \eta\target}{2}$, there exists a DP mechanism for which the gap in FNR $f_\mathsf{standard}(\alpha)$ obtained with standard calibration for $\varepsilon\target$ that ensures $\eta \leq \eta\target$, and FNR $f_\mathsf{adv}(\alpha)$ obtained with advantage calibration is lower bounded:
    \begin{align}
        \Delta \fnr(\fpr) \define f_\mathsf{standard}(\alpha) - f_\mathsf{adv}(\alpha) \geq \eta\target -  \delta\target + 2 \fpr \frac{\eta\target}{\eta\target - 1}.
    \end{align}
\end{restatable}
For example, if we aim to calibrate a mechanism to at most $\eta\target = 0.5$ (or, 75\% attack accuracy), we could potentially increase attack sensitivity by $\Delta \fnr(\fpr) \approx 30$ p.p. at FPR $\fpr = 0.1$ compared to standard calibration with $\delta\target = 10^{-5}$ (see the illustration in \cref{fig:adv-pitfalls}).
Note that the difference $\Delta \beta$ in \cref{stmt:adv-pitfalls} is an overestimate in practice: the increase in attack sensitivity can be significantly lower for mechanisms such as the Gaussian mechanism (see \cref{fig:adv-pitfalls-exact} in the Appendix).

\begin{figure}
\centering
\begin{subfigure}[t]{.45\textwidth}
    \centering
    \resizebox{.8\linewidth}{!}{
    \input{images/dpsgd_adv_calibration.pgf}
    }
    \caption{Calibrating noise to attack advantage significantly reduces the required noise scale compared to the standard approach. y axis is logarithmic.}
    \label{fig:quick-demo-adv-calibration}
\end{subfigure}\hspace{2em}
\begin{subfigure}[t]{.45\textwidth}
    \centering
    \resizebox{0.8\linewidth}{!}{
    \begin{tikzpicture}
        \begin{axis}[
            width=8cm,
            height=6.49cm,
            xlabel={Attack FPR, $\fpr$},
            ylabel={Attack FNR, $\fnr$},
            domain=0:1,
            xmin=0,
            xmax=1,
            ymin=0,
            ymax=1,
            samples=100, %
            legend pos=north east, %
            legend cell align={left},
            axis line style={draw opacity=0.2},
            trim axis left,
            trim axis right,
            enlargelimits=false,
            font={\sffamily},
            xtick={0, 0.5, 1},
            ytick={0, 0.5, 1},
            legend style={
              draw=lightgray,
              rounded corners,
            },
        ]
        \def\eps{0.9}
        \def\delta{0.05}
        \def\etacomputed{(exp(\eps) - 1 + 2*\delta) / (1 + exp(\eps))}
        \addlegendimage{empty legend}
        \addlegendentry{\hspace{.75cm}Method}

        \addplot[color=seabornblue, domain=0:1, line width=1.2] {max(0, 1 - \delta - exp(\eps) * x, exp(-\eps) * (1 - \delta - x)};
        \addlegendentry{Standard calibration}

        \addplot[color=seabornorange, domain=0:1 - \etacomputed, line width=1.2] {1 - \etacomputed - x)};
        \addlegendentry{Advantage calibration}

        \draw[<->] (axis cs:0.11, 0.65) -- node[left] {$\Delta \fnr$} (axis cs:0.11, 0.46);
        \end{axis}
    \end{tikzpicture}
    }
    \caption{Optimal calibration for advantage comes with a pitfall: it allows for $\Delta \beta$ higher attack power in the low FPR regime compared to standard calibration.
    }
    \label{fig:adv-pitfalls}

\end{subfigure}
\caption{Benefits and pitfalls of advantage calibration.}
\end{figure}

\subsection{Safer Choice: Calibration to FNR within a Given FPR Region}
\label{sec:err-rates-calibration}

In this section, we show how to calibrate the noise in any practical DP mechanism to a given minimum level of attack FNR $\fnr\target$ within an FPR region $\fpr \in [0, \fpr\target]$, which enables to avoid the pitfalls of advantage calibration. We base this notion of risk off the previous work~\cite{carlini2022membership,rezaei2021difficulty} which argued that MIAs are a relevant threat only when the achievable TPR $1 - \fnr$ is high at low FPR $\fpr$. We instantiate the calibration problem in \cref{eq:generic-calibration} as follows, assuming $\mechanism_\noiseparams(\cdot)$ satisfies $f_\noiseparams(\fpr)$-DP:
\vspace{-.2em}
\begin{align}\label{eq:f-region-calibration}
    \min_{\noiseparams \in \noiseparamspace} \: \noiseparams\text{ s.t. } \inf_{0 \leq \fpr \leq \fpr\target} f_\noiseparams(\fpr) \geq \fnr\target .
\end{align}
To solve \cref{eq:f-region-calibration}, we begin by showing that such calibration is in fact equivalent to requiring a given level of attack FNR $\fnr\target$ and FPR $\fpr\target$.
\begin{proposition}
For any $\fpr\target \geq 0, \fnr\target \geq 0$ such that $\fpr\target + \fnr\target \leq 1$, and any $f$-DP mechanism $\mechanism(\cdot)$:
    \begin{align}
    \inf_{0 \leq \fpr \leq \fpr\target} f(\fpr) \geq \fnr\target \text{ iff }
    f(\fpr\target) \geq \fnr\target.
    \end{align}
\end{proposition}
\vspace{-.8em}
This follows directly by monotonicity of the trade-off function $f$~\cite{dong2019gaussian}. The optimization problem becomes:
\vspace{-.8em}
\begin{align} \label{eq:f-calibration}
    \min_{\noiseparams \in \noiseparamspace} \: \noiseparams\text{ s.t. } f_{\noiseparams}(\fpr \target) \geq \fnr\target .
\end{align}
Unlike advantage calibration to $\eta\target$, the approach in \cref{eq:f-calibration} limits the adversary's capabilities without increasing the risk in the privacy-critical low-FPR regime, as we can explicitly control the acceptable attack sensitivity for a given low FPR.

To obtain $f_\noiseparams(\fpr)$, we use the PLRVs $X_\noiseparams, Y_\noiseparams$ along with \cref{stmt:plrv-to-risks} to compute $f = T(P,Q)$\footnote{In practice, we need to additionally symmetrize the trade-off curve due to the implementation details of the add/remove neighborhood relation in the \citet{doroshenko2022connect} accountant. See \cref{app:practical-considerations}.} (see \cref{alg:get-beta}), and solve \cref{eq:f-calibration} using binary search over $\noiseparams \in \noiseparamspace$. We provide the precise procedure in \cref{alg:direct-err-rates-calibration} in the Appendix. This approach \emph{guarantees} the desired level of risk:
\begin{restatable}{proposition}{fprcorrectness}
    Given PLRVs $(X_\noiseparams, Y_\noiseparams)$ of a discrete-valued dominating pair of a mechanism $\mechanism_\noiseparams(\cdot)$, choosing $\noiseparams^*$ using \cref{eq:f-calibration} and \cref{alg:get-beta} to compute $f_\noiseparams(\alpha)$ ensures $f_{\noiseparams^*}(\alpha\target) \geq \beta\target$.
    \label{stmt:f-calibration-correctness}
\end{restatable}

\begin{algorithm}[t]
\caption{Construct the trade-off curve using discrete privacy loss random variables $(X, Y)$}
\label{alg:get-beta}
\begin{algorithmic}[1]
\Require PMF $\Pr[X_\noiseparams = x_i]$ over grid $\{x_1, x_2, \ldots, x_k\}$ with $x_1 < x_2 < \ldots < x_k$
\Require PMF $\Pr[Y_\noiseparams = y_j]$ over grid $\{y_1, y_2, \ldots, y_l\}$ with $y_1 < y_2 < \ldots < y_l$
\Procedure{ComputeBeta}{$\noiseparams; \alpha\target; X_\noiseparams, Y_\noiseparams$}
\State \(t \leftarrow \min \{i \in \{0, 1, \ldots, k\} \mid \Pr[X_\noiseparams > x_{i}] \leq \fpr\target \}, \text{ where } x_0 \define -\infty \)
\State \(\gamma \leftarrow \frac{\fpr\target - \Pr[X_\noiseparams > x_{t}]}{\Pr[X_\noiseparams = x_{t}]}\)
\State \Return $f_\noiseparams(\fpr \target) = \Pr[Y_\noiseparams \leq x_{t}] - \gamma \Pr[Y_\noiseparams = x_{t}]$
\EndProcedure
\end{algorithmic}
\end{algorithm}

\subsection{Other Approaches to Trade-Off Curve Accounting}
In this section, we first contextualize the proposed method within existing work. Then, we discuss settings in which alternatives to PLRV-based procedures could be more suitable.

\paragraph{Benefits of PLRV-based Trade-Off Curve Accounting} Computational efficiency is important when estimating $f_\noiseparams(\alpha)$, as the calibration problem requires evaluating this function multiple times for different values of $\noiseparams$ as part of binary search. \cref{alg:get-beta} computes $f_\noiseparams(\alpha)$ for a single $\noiseparams$ in $\approx500$ms, enabling fast calibration, e.g., in $\approx1$ minute for DP-SGD with $T = \num{10000}$ steps on commodity hardware (see \cref{app:exp}). Existing methods for estimating $f_\noiseparams(\alpha)$, on the contrary, either provide weaker guarantees than \cref{stmt:f-calibration-correctness} or are substantially less efficient. 
In particular, \citet{dong2019gaussian} introduced $\mu$-GDP, an asymptotic expression for $f_\noiseparams(\alpha)$ as $T \rightarrow \infty$, that \emph{overestimates} privacy~\cite{gopi2021numerical}, and thus leads to mechanisms that do not satisfy the desired level of attack resilience when calibrating to it.
\citet{nasr2023tight,zheng2020sharp} introduced a discretization-based approach to approximate $f_\noiseparams(\alpha)$ (discussed next) that can be orders of magnitude less efficient than the direct estimation in \cref{alg:get-beta}, e.g., 1--6 minutes ($\approx$ 100--700$\times$ slower) for a single evaluation of $f_\noiseparams(\alpha)$ in the same setting as before, depending on the coarseness of discretization.

\paragraph{Calibration using Black-Box Accountants} Most DP mechanisms are accompanied by $(\varepsilon, \delta)$-DP accountants, i.e., methods to compute their privacy profile $\varepsilon_\noiseparams(\delta)$ or $\delta_\noiseparams(\varepsilon)$. Black-box access to these accountants enables to estimate $\eta_\noiseparams$ and $f_\noiseparams(\alpha)$. In particular, \cref{stmt:dp-to-adv} tells us that $(0, \delta)$-DP mechanisms bound advantage as $\eta \leq \delta$. Thus, advantage calibration can also be performed with any $\varepsilon_\noiseparams(\delta)$ accountant by calibrating noise to ensure $\varepsilon_\noiseparams(\eta\target) = 0$. Estimating $f_\noiseparams(\alpha)$, as mentioned previously, is less straightforward. Existing numeric approaches~\cite{nasr2023tight,zheng2020sharp} are equivalent to approximating \cref{eq:profiles-to-f} on a discrete grid over $\delta \in \{\delta_1, \ldots, \delta_u\}$. This requires $u$ calls to the accountant $\varepsilon_\noiseparams(\delta)$, thus quickly becomes inefficient for estimating $f_\noiseparams(\alpha)$ to high precision. We provide a detailed discussion of such black-box approaches in \cref{app:blackbox}.

\paragraph{Calibration of Mechanisms with Known Trade-Off Curves} An important feature of our calibration methods is that they enable calibration of mechanisms whose privacy profile is unknown in the exact form, e.g., DP-SGD for $T > 1$. Simpler mechanisms, such as the Gaussian mechanism, which are used for simpler statistical analyses, e.g., private mean estimation, admit exact analytical solutions to the calibration problems in \cref{eq:adv-calibration,eq:f-calibration}. In \cref{app:calibrating-specific-mechanisms}, we provide such solutions for the standard Gaussian mechanism, which enable efficient calibration without needing \cref{alg:get-beta}.

\section{Experiments}
\label{sec:experiments}

\begin{figure*}[t]
    \centering
    \resizebox{\textwidth}{!}{
    \input{images/dpsgd_err_rates_calibration.pgf}
    }
    \caption{Calibration to attack TPR (i.e., $1 - $FNR) significantly reduces the noise scale in low FPR regimes. Unlike calibration for attack advantage, this approach does not come with a deterioration of privacy for low FPR, as it directly targets this regime. }
    \label{fig:quick-demo-err-rates-calibration}
    \vspace{-.5em}
\end{figure*}

\begin{figure*}[t]
    \centering
    \resizebox{\textwidth}{!}{
    \input{images/gpt2_trade_off_curves.pgf}
    }
    \caption{Trade-off curves obtained via our method in \cref{alg:get-beta} provide a significantly tighter analysis of the attack risks, compared to the standard method of interpreting the privacy risk for a given $(\varepsilon, \delta)$ with fixed $\delta < \nicefrac{1}{n}$ via \cref{eq:dp-to-f}. The trade-off curves are shown for three runs of DP-SGD with different noise multipliers in the language modeling experiment with GPT-2. The dotted line \textcolor{gray}{-\,-} shows the trade-off curve which corresponds to perfect privacy.}
    \label{fig:gpt2-trade-off curves}
    \vspace{-.5em}
\end{figure*}

In this section, we empirically evaluate the utility improvement of our calibration method over traditional approaches. We do so in simulations as well as in realistic applications of DP-SGD. In \cref{app:exp}, we also evaluate the utility gain when performing simpler statistical analyses.

\paragraph{Simulations}
First, we demonstrate the noise reduction when calibrating the DP-SGD algorithm for given error rates using the setup in \cref{sec:quick-demo-adv-calibration}. We fix three low FPR values: $\fpr\target \in \{0.01, 0.05, 0.1\}$, and vary maximum attack sensitivity $1 - \fnr\target$ from $0.1$ to $0.5$ in each FPR regime. We show the results in \cref{fig:quick-demo-err-rates-calibration}. We observe a significant decrease in the noise scale for all values. Although the decrease is smaller than with calibration for advantage (see \cref{fig:quick-demo-adv-calibration}), calibrating directly for risk in the low FPR regime avoids the pitfall of advantage calibration: inadvertently increasing risk in this regime.

\paragraph{Language Modeling and Image Classification}
We showed that FPR/FNR calibration enables to significantly reduce the noise scale. Next, we study how much of this reduction in noise translates into actual utility improvement in downstream applications. We evaluate our method for calibrating noise in private deep learning on two tasks: text sentiment classification using the SST-2 dataset~\cite{socher2013recursive}, and image classification using the CIFAR-10 dataset~\cite{krizhevsky2009learning}.

For sentiment classification, we fine-tune GPT-2 (small)~\cite{radford2019language} using a DP version of LoRA~\cite{yu2021differentially}. For image classification, we follow the approach of \citet{tramer2021differentially} of training a convolutional neural network on top of ScatterNet features~\cite{oyallon2015deep} with DP-SGD~\cite{abadi2016deep}. See additional details in \cref{app:exp}. For each setting, by varying the noise scale, we obtain several models at different levels of privacy.
For each of the models we compute the guarantees in terms of TPR $1 - \fnr$ at three fixed levels of FPR $\fpr\target \in \{0.01, 0.05, 0.1\}$ that would be obtained under standard calibration, and using our \cref{alg:get-beta}.

\cref{fig:cifar10-err-rates-calibration} shows that FPR/FNR calibration significantly increases \emph{task} accuracy (a notion of utility; not to confuse with \emph{attack} accuracy, a notion of privacy risk) at the same level of $1 - \fnr$ for all values of $\fpr\target$. For instance, for GPT-2, we see the accuracy increase of 18.3 p.p. at the same level of privacy risk (top leftmost plot). To illustrate the reasons behind such a large difference between the methods, in \cref{fig:gpt2-trade-off curves}, we show the trade-off curves obtained with our \cref{alg:get-beta}, and with the standard method of deriving the FPR/FNR curve from a single $(\varepsilon, \delta)$ pair for a fixed $\delta < \nicefrac{1}{n}$ via \cref{eq:dp-to-f}. We can see that the latter approach drastically overestimates the attack risks, which translates to significantly higher noise and lower task accuracy when calibrating with standard calibration.

\section{Concluding Remarks}
\label{sec:conclusions}

In this work, we proposed novel methods for calibrating noise in differentially private learning targeting a given level of operational privacy risk: advantage and FPR/FNR of membership inference attacks. We introduced an accounting algorithm which directly and tightly estimates privacy guarantees in terms of $f$-DP, which characterizes these operational risks. Using simulations and end-to-end experiments on common use cases, we showed that our attack-aware noise calibration significantly decreases the required level of noise compared to the standard approach at the same level of operational risk. In the case of calibration for advantage, we also showed that the noise decrease could be harmful as it could allow for increased attack success in the low FPR regime compared to the standard approach, whereas calibration for a given level of FPR/FNR mitigates this issue. Next, we discuss limitations and possible directions for future work.

\paragraph{Choice of Target FPR/FNR} We leave open the question on how to choose the target FPR $\alpha\target$ and FNR $\beta\target$, e.g., whether standard significance levels in sciences such as $\alpha\target = 0.05$ are compatible with data protection regulation and norms. Further work is needed to develop concrete guidance on the choice of target FPR and FNR informed by legal and practical constraints.

\paragraph{Catastrophic Failures} It is possible to construct pathological DP mechanisms which admit catastrophic failures~\cite[see, e.g.,][]{dpfy}, i.e., mechanisms which allow non-trivial attack TPR at FPR $\alpha = 0$ so that their trade-off curve is such that $T(M(S), M(S'))(0) < 1$ for some $S \simeq S'$. A classical example in the context of private data release is a mechanism that releases a data record in the clear with probability $\delta > 0$, in which case we have $T(M(S), M(S'))(0) = 1 - \delta$. See the proof of \cref{stmt:adv-pitfalls} in \cref{app:proofs} for a concrete construction. In the case that such a pathological mechanism is used in practice, one should use standard calibration to $(\varepsilon, \delta)$ with $\delta \ll \nicefrac{1}{n}$ to directly limit the chance of catastrophic failures. Fortunately, practical mechanisms such as DP-SGD do not admit catastrophic failures, as they ensure $T(M(S), M(S'))(0) = 1$. 

\paragraph{Tight Bounds for Privacy Auditing} Multiple prior works on auditing the privacy properties of ML algorithms~\cite{nasr2021adversary,liu2021generalization,jayaraman2019evaluating,erlingsson2019we} used conversions between $(\varepsilon, \delta)$ and operational risks like in \cref{stmt:dp-to-f}, which we have shown to significantly overestimate the actual risks. Beyond calibrating noise, our methods provide bounds on attack success rates for audits in a more precise and computationally efficient way than a recent similar approach by \citet{nasr2023tight}.

\paragraph{Accounting in the Relaxed Threat Models} Although we have focused on DP, our methods apply to any notion of privacy that is also formalized as a hypothesis test. In particular, our method can be used as is to compute privacy guarantees of DP-SGD in a relaxed threat model (RTM) proposed by \citet{kaissis2023optimal}. Previously, there was no efficient method for accounting in the RTM.

\paragraph{Applications Beyond Privacy} Our method can be applied to ensure provable generalization guarantees in deep learning. Indeed, prior work has shown that advantage $\eta$ bounds generalization gaps of ML models~\cite{kulynych2022disparate,kulynych2022what}. 
Thus, even though advantage calibration can exacerbate certain risks, it can be a useful tool for ensuring a desired level of generalization in models that usually do not come with non-vacuous generalization guarantees, e.g., deep neural networks.

\clearpage

\section*{Acknowledgements}
\vspace{-.2em}
The authors thank the anonymous referees for their feedback,
which considerably helped to improve the quality of the paper, and Priyanka Nanayakkara for the helpful suggestions. This paper is based upon work supported by the U.S. National Science Foundation under awards CIF-1900750, CIF-2231707, CIF-2312667, and FAI-2040880, U.S. Department of Energy Award No. DE-SC0022158, and by the EU Innovative Health Initiative Joint Undertaking (IHI JU) under grant agreement No. 101172872. GK received support from the German Federal Ministry of Education and Research and the Bavarian State Ministry for Science and the Arts under the Munich Centre for Machine Learning (MCML), from the German Ministry of Education and Research and the Medical Informatics Initiative as part of the PrivateAIM Project, from the Bavarian Collaborative Research Project PRIPREKI of the Free State of Bavaria Funding Programme ``Artificial Intelligence -- Data Science'', and from the German Academic Exchange Service (DAAD) under the Kondrad Zuse School of Excellence for Reliable AI (RelAI).

\bibliographystyle{plainnat}
\bibliography{main}

\begin{thebibliography}{69}
\providecommand{\natexlab}[1]{#1}
\providecommand{\url}[1]{\texttt{#1}}
\expandafter\ifx\csname urlstyle\endcsname\relax
  \providecommand{\doi}[1]{doi: #1}\else
  \providecommand{\doi}{doi: \begingroup \urlstyle{rm}\Url}\fi

\bibitem[Abadi et~al.(2016)Abadi, Chu, Goodfellow, McMahan, Mironov, Talwar,
  and Zhang]{abadi2016deep}
Martin Abadi, Andy Chu, Ian Goodfellow, H~Brendan McMahan, Ilya Mironov, Kunal
  Talwar, and Li~Zhang.
\newblock Deep learning with differential privacy.
\newblock In \emph{Proceedings of the 2016 ACM SIGSAC conference on computer
  and communications security}, 2016.

\bibitem[Abowd and Schmutte(2015)]{abowd2015revisiting}
John~M Abowd and Ian~M Schmutte.
\newblock Revisiting the economics of privacy: Population statistics and
  confidentiality protection as public goods.
\newblock \emph{Economics}, \penalty0 (1/20), 2015.

\bibitem[Alghamdi et~al.(2023)Alghamdi, Gomez, Asoodeh, Calmon, Kosut, and
  Sankar]{alghamdi2022saddle}
Wael Alghamdi, Juan~Felipe Gomez, Shahab Asoodeh, Flavio Calmon, Oliver Kosut,
  and Lalitha Sankar.
\newblock The saddle-point method in differential privacy.
\newblock In \emph{International Conference on Machine Learning}, pages
  508--528. PMLR, 2023.

\bibitem[Balle and Wang(2018)]{balle2018improving}
Borja Balle and Yu-Xiang Wang.
\newblock Improving the gaussian mechanism for differential privacy: Analytical
  calibration and optimal denoising.
\newblock In \emph{International Conference on Machine Learning}. PMLR, 2018.

\bibitem[Balle et~al.(2022)Balle, Cherubin, and Hayes]{balle2022reconstructing}
Borja Balle, Giovanni Cherubin, and Jamie Hayes.
\newblock Reconstructing training data with informed adversaries.
\newblock In \emph{2022 IEEE Symposium on Security and Privacy (SP)}. IEEE,
  2022.

\bibitem[Becker and Kohavi(1996)]{adult}
Barry Becker and Ronny Kohavi.
\newblock {Adult}.
\newblock UCI Machine Learning Repository, 1996.
\newblock {DOI}: https://doi.org/10.24432/C5XW20.

\bibitem[Carlini et~al.(2022)Carlini, Chien, Nasr, Song, Terzis, and
  Tramer]{carlini2022membership}
Nicholas Carlini, Steve Chien, Milad Nasr, Shuang Song, Andreas Terzis, and
  Florian Tramer.
\newblock Membership inference attacks from first principles.
\newblock In \emph{2022 IEEE Symposium on Security and Privacy (SP)}, 2022.

\bibitem[Chaudhuri et~al.(2011)Chaudhuri, Monteleoni, and
  Sarwate]{chaudhuri2011differentially}
Kamalika Chaudhuri, Claire Monteleoni, and Anand~D Sarwate.
\newblock Differentially private empirical risk minimization.
\newblock \emph{Journal of Machine Learning Research}, 2011.

\bibitem[Cherubin et~al.(2024)Cherubin, K{\"{o}}pf, Paverd, Tople, Wutschitz,
  and B{\'{e}}guelin]{cherubin2024closed}
Giovanni Cherubin, Boris K{\"{o}}pf, Andrew Paverd, Shruti Tople, Lukas
  Wutschitz, and Santiago~Zanella B{\'{e}}guelin.
\newblock Closed-form bounds for {DP-SGD} against record-level inference.
\newblock In \emph{33rd {USENIX} Security Symposium ({USENIX} Security 2024)},
  2024.

\bibitem[Cummings et~al.(2021)Cummings, Kaptchuk, and
  Redmiles]{cummings2021need}
Rachel Cummings, Gabriel Kaptchuk, and Elissa~M Redmiles.
\newblock ``{I need a better description}''': An investigation into user
  expectations for differential privacy.
\newblock In \emph{Proceedings of the 2021 ACM SIGSAC Conference on Computer
  and Communications Security}, 2021.

\bibitem[Dong et~al.(2022)Dong, Roth, and Su]{dong2019gaussian}
Jinshuo Dong, Aaron Roth, and Weijie~J Su.
\newblock Gaussian differential privacy.
\newblock \emph{Journal of the Royal Statistical Society Series B: Statistical
  Methodology}, 2022.

\bibitem[Doroshenko et~al.(2022)Doroshenko, Ghazi, Kamath, Kumar, and
  Manurangsi]{doroshenko2022connect}
Vadym Doroshenko, Badih Ghazi, Pritish Kamath, Ravi Kumar, and Pasin
  Manurangsi.
\newblock Connect the dots: Tighter discrete approximations of privacy loss
  distributions.
\newblock \emph{Proceedings on Privacy Enhancing Technologies}, 2022.

\bibitem[Dwork and Rothblum(2016)]{dwork2016concentrated}
Cynthia Dwork and Guy~N Rothblum.
\newblock Concentrated differential privacy.
\newblock \emph{arXiv preprint arXiv:1603.01887}, 2016.

\bibitem[Dwork et~al.(2006)Dwork, McSherry, Nissim, and
  Smith]{dwork2006calibrating}
Cynthia Dwork, Frank McSherry, Kobbi Nissim, and Adam Smith.
\newblock Calibrating noise to sensitivity in private data analysis.
\newblock In \emph{Proceedings of the Theory of Cryptography Conference}, 2006.

\bibitem[Dwork et~al.(2014)Dwork, Roth, et~al.]{dwork2014algorithmic}
Cynthia Dwork, Aaron Roth, et~al.
\newblock The algorithmic foundations of differential privacy.
\newblock \emph{Foundations and Trends in Theoretical Computer Science}, 2014.

\bibitem[Erlingsson et~al.(2019)Erlingsson, Mironov, Raghunathan, and
  Song]{erlingsson2019we}
{\'U}lfar Erlingsson, Ilya Mironov, Ananth Raghunathan, and Shuang Song.
\newblock That which we call private.
\newblock \emph{arXiv preprint arXiv:1908.03566}, 2019.

\bibitem[Franzen et~al.(2022)Franzen, Nu{\~n}ez~von Voigt, S{\"o}rries,
  Tschorsch, and M{\"u}ller-Birn]{franzen2022private}
Daniel Franzen, Saskia Nu{\~n}ez~von Voigt, Peter S{\"o}rries, Florian
  Tschorsch, and Claudia M{\"u}ller-Birn.
\newblock Am i private and if so, how many? communicating privacy guarantees of
  differential privacy with risk communication formats.
\newblock In \emph{Proceedings of the 2022 ACM SIGSAC Conference on Computer
  and Communications Security}, 2022.

\bibitem[Fredrikson et~al.(2015)Fredrikson, Jha, and
  Ristenpart]{fredrikson2015model}
Matt Fredrikson, Somesh Jha, and Thomas Ristenpart.
\newblock Model inversion attacks that exploit confidence information and basic
  countermeasures.
\newblock In \emph{Proceedings of the ACM SIGSAC conference on computer and
  communications security}, 2015.

\bibitem[Gaboardi et~al.(2020)Gaboardi, Hay, and
  Vadhan]{gaboardi2020programming}
Marco Gaboardi, Michael Hay, and Salil Vadhan.
\newblock A programming framework for {OpenDP}.
\newblock \emph{Manuscript, May}, 2020.

\bibitem[Ghazi and Issa(2023)]{ghazi2023total}
Elena Ghazi and Ibrahim Issa.
\newblock Total variation with differential privacy: Tighter composition and
  asymptotic bounds.
\newblock In \emph{IEEE International Symposium on Information Theory (ISIT)},
  2023.

\bibitem[Gopi et~al.(2021)Gopi, Lee, and Wutschitz]{gopi2021numerical}
Sivakanth Gopi, Yin~Tat Lee, and Lukas Wutschitz.
\newblock Numerical composition of differential privacy.
\newblock \emph{Advances in Neural Information Processing Systems ({NeurIPS})},
  2021.

\bibitem[Harris et~al.(2020)Harris, Millman, van~der Walt, Gommers, Virtanen,
  Cournapeau, Wieser, Taylor, Berg, Smith, Kern, Picus, Hoyer, van Kerkwijk,
  Brett, Haldane, del R{\'{i}}o, Wiebe, Peterson, G{\'{e}}rard-Marchant,
  Sheppard, Reddy, Weckesser, Abbasi, Gohlke, and Oliphant]{numpy}
Charles~R. Harris, K.~Jarrod Millman, St{\'{e}}fan~J. van~der Walt, Ralf
  Gommers, Pauli Virtanen, David Cournapeau, Eric Wieser, Julian Taylor,
  Sebastian Berg, Nathaniel~J. Smith, Robert Kern, Matti Picus, Stephan Hoyer,
  Marten~H. van Kerkwijk, Matthew Brett, Allan Haldane, Jaime~Fern{\'{a}}ndez
  del R{\'{i}}o, Mark Wiebe, Pearu Peterson, Pierre G{\'{e}}rard-Marchant,
  Kevin Sheppard, Tyler Reddy, Warren Weckesser, Hameer Abbasi, Christoph
  Gohlke, and Travis~E. Oliphant.
\newblock Array programming with {NumPy}.
\newblock \emph{Nature}, 2020.

\bibitem[Hayes et~al.(2024)Hayes, Balle, and Mahloujifar]{hayes2024bounding}
Jamie Hayes, Borja Balle, and Saeed Mahloujifar.
\newblock Bounding training data reconstruction in {DP-SGD}.
\newblock \emph{Advances in Neural Information Processing Systems}, 2024.

\bibitem[Hsu et~al.(2014)Hsu, Gaboardi, Haeberlen, Khanna, Narayan, Pierce, and
  Roth]{hsu2014differential}
Justin Hsu, Marco Gaboardi, Andreas Haeberlen, Sanjeev Khanna, Arjun Narayan,
  Benjamin~C Pierce, and Aaron Roth.
\newblock Differential privacy: An economic method for choosing epsilon.
\newblock In \emph{2014 IEEE 27th Computer Security Foundations Symposium},
  2014.

\bibitem[Hu et~al.(2021)Hu, Wallis, Allen-Zhu, Li, Wang, Wang, Chen,
  et~al.]{hu2021lora}
Edward~J Hu, Phillip Wallis, Zeyuan Allen-Zhu, Yuanzhi Li, Shean Wang, Lu~Wang,
  Weizhu Chen, et~al.
\newblock Lora: Low-rank adaptation of large language models.
\newblock In \emph{International Conference on Learning Representations}, 2021.

\bibitem[Izzo et~al.(2024)Izzo, Yoon, Arik, and Zou]{izzo2022provable}
Zachary Izzo, Jinsung Yoon, Sercan~O Arik, and James Zou.
\newblock Provable membership inference privacy.
\newblock \emph{Transactions on Machine Learning Research}, 2024.

\bibitem[Jayaraman and Evans(2019)]{jayaraman2019evaluating}
Bargav Jayaraman and David Evans.
\newblock Evaluating differentially private machine learning in practice.
\newblock In \emph{28th USENIX Security Symposium (USENIX Security 19)}, 2019.

\bibitem[Jayaraman et~al.(2021)Jayaraman, Wang, Knipmeyer, Gu, and
  Evans]{jayaraman2021revisiting}
Bargav Jayaraman, Lingxiao Wang, Katherine Knipmeyer, Quanquan Gu, and David
  Evans.
\newblock Revisiting membership inference under realistic assumptions.
\newblock \emph{Proceedings on Privacy Enhancing Technologies}, 2021.

\bibitem[Jin et~al.(2023)Jin, Su, Zhong, Zhang, Quek, and Dai]{discrete_fdp}
Richeng Jin, Zhonggen Su, Caijun Zhong, Zhaoyang Zhang, Tony Quek, and Huaiyu
  Dai.
\newblock Breaking the communication-privacy-accuracy tradeoff with
  f-differential privacy.
\newblock \emph{Advances in Neural Information Processing Systems ({NeurIPS})},
  2023.

\bibitem[Kairouz et~al.(2015)Kairouz, Oh, and
  Viswanath]{kairouz2015composition}
Peter Kairouz, Sewoong Oh, and Pramod Viswanath.
\newblock The composition theorem for differential privacy.
\newblock In \emph{International Conference on Machine Learning}. PMLR, 2015.

\bibitem[Kaissis et~al.(2023{\natexlab{a}})Kaissis, Hayes, Ziller, and
  Rueckert]{kaissis2023bounding}
Georgios Kaissis, Jamie Hayes, Alexander Ziller, and Daniel Rueckert.
\newblock Bounding data reconstruction attacks with the hypothesis testing
  interpretation of differential privacy.
\newblock \emph{arXiv preprint arXiv:2307.03928}, 2023{\natexlab{a}}.

\bibitem[Kaissis et~al.(2023{\natexlab{b}})Kaissis, Ziller, Kolek, Riess, and
  Rueckert]{kaissis2023optimal}
Georgios Kaissis, Alexander Ziller, Stefan Kolek, Anneliese Riess, and Daniel
  Rueckert.
\newblock Optimal privacy guarantees for a relaxed threat model: Addressing
  sub-optimal adversaries in differentially private machine learning.
\newblock \emph{Advances in Neural Information Processing Systems ({NeurIPS})},
  2023{\natexlab{b}}.

\bibitem[Kluyver et~al.(2016)Kluyver, Ragan-Kelley, P{\'e}rez, Granger,
  Bussonnier, Frederic, Kelley, Hamrick, Grout, Corlay, Ivanov, Avila, Abdalla,
  Willing, and development team]{jupyter}
Thomas Kluyver, Benjamin Ragan-Kelley, Fernando P{\'e}rez, Brian Granger,
  Matthias Bussonnier, Jonathan Frederic, Kyle Kelley, Jessica Hamrick, Jason
  Grout, Sylvain Corlay, Paul Ivanov, Dami{\'a}n Avila, Safia Abdalla, Carol
  Willing, and Jupyter development team.
\newblock Jupyter notebooks - a publishing format for reproducible
  computational workflows.
\newblock In \emph{Positioning and Power in Academic Publishing: Players,
  Agents and Agendas}. IOS Press, 2016.

\bibitem[Krizhevsky et~al.(2009)Krizhevsky, Hinton,
  et~al.]{krizhevsky2009learning}
Alex Krizhevsky, Geoffrey Hinton, et~al.
\newblock Learning multiple layers of features from tiny images.
\newblock \emph{Technical report}, 2009.

\bibitem[Kulynych et~al.(2022{\natexlab{a}})Kulynych, Yaghini, Cherubin, Veale,
  and Troncoso]{kulynych2022disparate}
Bogdan Kulynych, Mohammad Yaghini, Giovanni Cherubin, Michael Veale, and
  Carmela Troncoso.
\newblock Disparate vulnerability to membership inference attacks.
\newblock \emph{Proceedings on Privacy Enhancing Technologies},
  2022{\natexlab{a}}.

\bibitem[Kulynych et~al.(2022{\natexlab{b}})Kulynych, Yang, Yu, Blasiok, and
  Nakkiran]{kulynych2022what}
Bogdan Kulynych, Yao-Yuan Yang, Yaodong Yu, Jaroslaw Blasiok, and Preetum
  Nakkiran.
\newblock What you see is what you get: Principled deep learning via
  distributional generalization.
\newblock \emph{Advances in Neural Information Processing Systems ({NeurIPS})},
  2022{\natexlab{b}}.

\bibitem[Leemann et~al.(2024)Leemann, Pawelczyk, and
  Kasneci]{leemann2024gaussian}
Tobias Leemann, Martin Pawelczyk, and Gjergji Kasneci.
\newblock Gaussian membership inference privacy.
\newblock \emph{Advances in Neural Information Processing Systems}, 36, 2024.

\bibitem[Lehmann and Romano.(2006)]{Romano2006}
Erich~L Lehmann and Joseph~P Romano.
\newblock \emph{Testing statistical hypotheses}.
\newblock Springer Science \& Business Media, 2006.

\bibitem[Liu et~al.(2021)Liu, Oya, and Kerschbaum]{liu2021generalization}
Jiaxiang Liu, Simon Oya, and Florian Kerschbaum.
\newblock Generalization techniques empirically outperform differential privacy
  against membership inference.
\newblock \emph{arXiv preprint arXiv:2110.05524}, 2021.

\bibitem[Mahloujifar et~al.(2022)Mahloujifar, Sablayrolles, Cormode, and
  Jha]{mahloujifar2022optimal}
Saeed Mahloujifar, Alexandre Sablayrolles, Graham Cormode, and Somesh Jha.
\newblock Optimal membership inference bounds for adaptive composition of
  sampled gaussian mechanisms.
\newblock \emph{arXiv preprint arXiv:2204.06106}, 2022.

\bibitem[Mehner et~al.(2021)Mehner, von Voigt, and
  Tschorsch]{mehner2021towards}
Luise Mehner, Saskia~Nu{\~n}ez von Voigt, and Florian Tschorsch.
\newblock Towards explaining epsilon: A worst-case study of differential
  privacy risks.
\newblock In \emph{2021 IEEE European Symposium on Security and Privacy
  Workshops (EuroS\&PW)}, 2021.

\bibitem[Nanayakkara et~al.(2022)Nanayakkara, Bater, He, Hullman, and
  Rogers]{nanayakkara2022visualizing}
Priyanka Nanayakkara, Johes Bater, Xi~He, Jessica Hullman, and Jennie Rogers.
\newblock Visualizing privacy-utility trade-offs in differentially private data
  releases.
\newblock \emph{Proceedings on Privacy Enhancing Technologies}, 2:\penalty0
  601--618, 2022.

\bibitem[Nanayakkara et~al.(2023)Nanayakkara, Smart, Cummings, Kaptchuk, and
  Redmiles]{nanayakkara2023chances}
Priyanka Nanayakkara, Mary~Anne Smart, Rachel Cummings, Gabriel Kaptchuk, and
  Elissa~M Redmiles.
\newblock What are the chances? explaining the epsilon parameter in
  differential privacy.
\newblock In \emph{32nd USENIX Security Symposium (USENIX Security 23)}, 2023.

\bibitem[Nasr et~al.(2021)Nasr, Songi, Thakurta, Papernot, and
  Carlini]{nasr2021adversary}
Milad Nasr, Shuang Songi, Abhradeep Thakurta, Nicolas Papernot, and Nicholas
  Carlini.
\newblock Adversary instantiation: Lower bounds for differentially private
  machine learning.
\newblock In \emph{IEEE Symposium on security and privacy (SP)}, 2021.

\bibitem[Nasr et~al.(2023)Nasr, Hayes, Steinke, Balle, Tram{\`e}r, Jagielski,
  Carlini, and Terzis]{nasr2023tight}
Milad Nasr, Jamie Hayes, Thomas Steinke, Borja Balle, Florian Tram{\`e}r,
  Matthew Jagielski, Nicholas Carlini, and Andreas Terzis.
\newblock Tight auditing of differentially private machine learning.
\newblock In \emph{32nd USENIX Security Symposium (USENIX Security 23)}, 2023.

\bibitem[Near et~al.(2023)Near, Darais, Lefkovitz, Howarth,
  et~al.]{near2023guidelines}
Joseph~P Near, David Darais, Naomi Lefkovitz, Gary Howarth, et~al.
\newblock Guidelines for evaluating differential privacy guarantees.
\newblock \emph{National Institute of Standards and Technology, Tech. Rep},
  2023.

\bibitem[Nissim et~al.(2014)Nissim, Vadhan, and Xiao]{nissim2014redrawing}
Kobbi Nissim, Salil Vadhan, and David Xiao.
\newblock Redrawing the boundaries on purchasing data from privacy-sensitive
  individuals.
\newblock In \emph{Proceedings of the conference on Innovations in theoretical
  computer science}, 2014.

\bibitem[Oyallon and Mallat(2015)]{oyallon2015deep}
Edouard Oyallon and St{\'e}phane Mallat.
\newblock Deep roto-translation scattering for object classification.
\newblock In \emph{Proceedings of the IEEE Conference on Computer Vision and
  Pattern Recognition}, 2015.

\bibitem[pandas~development team(2020)]{pandas}
The pandas~development team.
\newblock pandas-dev/pandas: Pandas, 2020.

\bibitem[Paszke et~al.(2019)Paszke, Gross, Massa, Lerer, Bradbury, Chanan,
  Killeen, Lin, Gimelshein, Antiga, et~al.]{pytorch}
Adam Paszke, Sam Gross, Francisco Massa, Adam Lerer, James Bradbury, Gregory
  Chanan, Trevor Killeen, Zeming Lin, Natalia Gimelshein, Luca Antiga, et~al.
\newblock Pytorch: An imperative style, high-performance deep learning library.
\newblock In \emph{Advances in Neural Information Processing Systems
  ({NeurIPS})}, 2019.

\bibitem[Ponomareva et~al.(2023)Ponomareva, Hazimeh, Kurakin, Xu, Denison,
  McMahan, Vassilvitskii, Chien, and Thakurta]{dpfy}
Natalia Ponomareva, Hussein Hazimeh, Alex Kurakin, Zheng Xu, Carson Denison,
  H.~Brendan McMahan, Sergei Vassilvitskii, Steve Chien, and Abhradeep~Guha
  Thakurta.
\newblock How to {DP-fy} {ML}: A practical guide to machine learning with
  differential privacy.
\newblock \emph{Journal of Artificial Intelligence Research}, 2023.

\bibitem[Radford et~al.(2019)Radford, Wu, Child, Luan, Amodei, Sutskever,
  et~al.]{radford2019language}
Alec Radford, Jeffrey Wu, Rewon Child, David Luan, Dario Amodei, Ilya
  Sutskever, et~al.
\newblock Language models are unsupervised multitask learners.
\newblock \emph{OpenAI blog}, 2019.

\bibitem[Rezaei and Liu(2021)]{rezaei2021difficulty}
Shahbaz Rezaei and Xin Liu.
\newblock On the difficulty of membership inference attacks.
\newblock In \emph{Proceedings of the IEEE/CVF Conference on Computer Vision
  and Pattern Recognition}, 2021.

\bibitem[Shokri et~al.(2017)Shokri, Stronati, Song, and
  Shmatikov]{shokri2017membership}
Reza Shokri, Marco Stronati, Congzheng Song, and Vitaly Shmatikov.
\newblock Membership inference attacks against machine learning models.
\newblock In \emph{IEEE symposium on security and privacy (SP)}. IEEE, 2017.

\bibitem[Socher et~al.(2013)Socher, Perelygin, Wu, Chuang, Manning, Ng, and
  Potts]{socher2013recursive}
Richard Socher, Alex Perelygin, Jean Wu, Jason Chuang, Christopher~D. Manning,
  Andrew Ng, and Christopher Potts.
\newblock Recursive deep models for semantic compositionality over a sentiment
  treebank.
\newblock In \emph{Proceedings of the Conference on Empirical Methods in
  Natural Language Processing}, 2013.

\bibitem[Tramer and Boneh(2021)]{tramer2021differentially}
Florian Tramer and Dan Boneh.
\newblock Differentially private learning needs better features (or much more
  data).
\newblock In \emph{International Conference on Learning Representations}, 2021.

\bibitem[Virtanen et~al.(2020)Virtanen, Gommers, Oliphant, Haberland, Reddy,
  Cournapeau, Burovski, Peterson, Weckesser, Bright, {van der Walt}, Brett,
  Wilson, Millman, Mayorov, Nelson, Jones, Kern, Larson, Carey, Polat, Feng,
  Moore, {VanderPlas}, Laxalde, Perktold, Cimrman, Henriksen, Quintero, Harris,
  Archibald, Ribeiro, Pedregosa, {van Mulbregt}, and {SciPy 1.0
  Contributors}]{scipy}
Pauli Virtanen, Ralf Gommers, Travis~E. Oliphant, Matt Haberland, Tyler Reddy,
  David Cournapeau, Evgeni Burovski, Pearu Peterson, Warren Weckesser, Jonathan
  Bright, St{\'e}fan~J. {van der Walt}, Matthew Brett, Joshua Wilson, K.~Jarrod
  Millman, Nikolay Mayorov, Andrew R.~J. Nelson, Eric Jones, Robert Kern, Eric
  Larson, C~J Carey, {\.I}lhan Polat, Yu~Feng, Eric~W. Moore, Jake
  {VanderPlas}, Denis Laxalde, Josef Perktold, Robert Cimrman, Ian Henriksen,
  E.~A. Quintero, Charles~R. Harris, Anne~M. Archibald, Ant{\^o}nio~H. Ribeiro,
  Fabian Pedregosa, Paul {van Mulbregt}, and {SciPy 1.0 Contributors}.
\newblock {{SciPy} 1.0: Fundamental Algorithms for Scientific Computing in
  Python}.
\newblock \emph{Nature Methods}, 2020.

\bibitem[Wang et~al.(2018)Wang, Singh, Michael, Hill, Levy, and
  Bowman]{wang2018glue}
Alex Wang, Amanpreet Singh, Julian Michael, Felix Hill, Omer Levy, and Samuel
  Bowman.
\newblock Glue: A multi-task benchmark and analysis platform for natural
  language understanding.
\newblock In \emph{Proceedings of the 2018 EMNLP Workshop BlackboxNLP:
  Analyzing and Interpreting Neural Networks for NLP}, 2018.

\bibitem[Waskom(2021)]{seaborn}
Michael~L. Waskom.
\newblock seaborn: statistical data visualization.
\newblock \emph{Journal of Open Source Software}, 6\penalty0 (60), 2021.
\newblock \doi{10.21105/joss.03021}.
\newblock URL \url{https://doi.org/10.21105/joss.03021}.

\bibitem[Wasserman and Zhou(2010)]{wasserman2010statistical}
Larry Wasserman and Shuheng Zhou.
\newblock A statistical framework for differential privacy.
\newblock \emph{Journal of the American Statistical Association}, 2010.

\bibitem[Wolf et~al.(2019)Wolf, Debut, Sanh, Chaumond, Delangue, Moi, Cistac,
  Rault, Louf, Funtowicz, et~al.]{huggingface}
Thomas Wolf, Lysandre Debut, Victor Sanh, Julien Chaumond, Clement Delangue,
  Anthony Moi, Pierric Cistac, Tim Rault, R{\'e}mi Louf, Morgan Funtowicz,
  et~al.
\newblock Huggingface's transformers: State-of-the-art natural language
  processing.
\newblock \emph{arXiv preprint arXiv:1910.03771}, 2019.

\bibitem[Wood et~al.(2018)Wood, Altman, Bembenek, Bun, Gaboardi, Honaker,
  Nissim, O'Brien, Steinke, and Vadhan]{wood2018differential}
Alexandra Wood, Micah Altman, Aaron Bembenek, Mark Bun, Marco Gaboardi, James
  Honaker, Kobbi Nissim, David~R O'Brien, Thomas Steinke, and Salil Vadhan.
\newblock Differential privacy: A primer for a non-technical audience.
\newblock \emph{Vand. J. Ent. \& Tech. L.}, 2018.

\bibitem[Wutschitz et~al.(2022)Wutschitz, Inan, and Manoel]{dp-transformers}
Lukas Wutschitz, Huseyin~A. Inan, and Andre Manoel.
\newblock dp-transformers: Training transformer models with differential
  privacy.
\newblock
  \url{https://www.microsoft.com/en-us/research/project/dp-transformers},
  August 2022.

\bibitem[Yeom et~al.(2018)Yeom, Giacomelli, Fredrikson, and
  Jha]{yeom2018privacy}
Samuel Yeom, Irene Giacomelli, Matt Fredrikson, and Somesh Jha.
\newblock Privacy risk in machine learning: Analyzing the connection to
  overfitting.
\newblock In \emph{2018 IEEE 31st Computer Security Foundations Symposium
  (CSF)}. IEEE, 2018.

\bibitem[Yousefpour et~al.(2021)Yousefpour, Shilov, Sablayrolles, Testuggine,
  Prasad, Malek, Nguyen, Ghosh, Bharadwaj, Zhao, Cormode, and Mironov]{opacus}
Ashkan Yousefpour, Igor Shilov, Alexandre Sablayrolles, Davide Testuggine,
  Karthik Prasad, Mani Malek, John Nguyen, Sayan Ghosh, Akash Bharadwaj,
  Jessica Zhao, Graham Cormode, and Ilya Mironov.
\newblock Opacus: {U}ser-friendly differential privacy library in {PyTorch}.
\newblock \emph{arXiv preprint arXiv:2109.12298}, 2021.

\bibitem[Yu et~al.(2021)Yu, Naik, Backurs, Gopi, Inan, Kamath, Kulkarni, Lee,
  Manoel, Wutschitz, et~al.]{yu2021differentially}
Da~Yu, Saurabh Naik, Arturs Backurs, Sivakanth Gopi, Huseyin~A Inan, Gautam
  Kamath, Janardhan Kulkarni, Yin~Tat Lee, Andre Manoel, Lukas Wutschitz,
  et~al.
\newblock Differentially private fine-tuning of language models.
\newblock In \emph{International Conference on Learning Representations}, 2021.

\bibitem[Zheng et~al.(2020)Zheng, Dong, Long, and Su]{zheng2020sharp}
Qinqing Zheng, Jinshuo Dong, Qi~Long, and Weijie Su.
\newblock Sharp composition bounds for gaussian differential privacy via
  {Edgeworth} expansion.
\newblock In \emph{International Conference on Machine Learning}. PMLR, 2020.

\bibitem[Zhu et~al.(2022{\natexlab{a}})Zhu, Dong, and Wang]{characteristic}
Yuqing Zhu, Jinshuo Dong, and Yu-Xiang Wang.
\newblock Optimal accounting of differential privacy via characteristic
  function.
\newblock In \emph{Proceedings of The 25th International Conference on
  Artificial Intelligence and Statistics}, 2022{\natexlab{a}}.

\bibitem[Zhu et~al.(2022{\natexlab{b}})Zhu, Dong, and Wang]{zhu2022optimal}
Yuqing Zhu, Jinshuo Dong, and Yu-Xiang Wang.
\newblock Optimal accounting of differential privacy via characteristic
  function.
\newblock In \emph{International Conference on Artificial Intelligence and
  Statistics}. PMLR, 2022{\natexlab{b}}.

\end{thebibliography}

\appendix

\clearpage

\begin{table}[H]
\caption{Notation summary}
\label{tab:notation}
\vspace{.2em}
\resizebox{\linewidth}{!}{
\begin{tabular}{lll}
\textbf{Symbol} & \textbf{Description} & \textbf{Reference} \\
\midrule
$z \in \mathbb{D}$ & Data record & \\
$S \in 2^\mathbb{D}$ & Dataset of records & \\
$S \simeq S'$ & Adjacency relation of neighboring datasets & \\
$M_\noiseparams: 2^\mathbb{D} \to \Theta$ & Privacy-preserving mechanism & \\
$\noiseparams \in \noiseparamspace$ & Noise parameter of mechanism $M(S)$ & \\
$D_\gamma(M(S)~\|~M(S')), \gamma \geq 0$ & Hockeystick divergence & \cref{eq:hockeystick-def} \\
$\varepsilon \in (0, \infty), \delta \in [0,1]$ & Privacy parameters in differential privacy & \cref{def:dp} \\
$\varepsilon_\noiseparams :[0,1] \rightarrow \sR$ & Privacy profile curve $\varepsilon_\noiseparams(\delta) $& \cref{def:profile} \\
$\delta_\noiseparams: \sR \rightarrow [0, 1]$ & Privacy profile curve $\delta_\noiseparams(\varepsilon) $& \cref{def:profile} \\
$\phi : \Theta \to [0,1]$ & Membership inference hypothesis test & \\
$\alpha_\phi \in [0,1]$ & False positive rate (FPR) of attack $\phi(\theta)$ & \\
$\beta_\phi \in [0,1]$ & False negative rate (FNR) of attack $\phi(\theta)$ & \\
$\eta \in [0,1]$ & Maximal advantage across attacks against mechanism $M(S)$ & \cref{eq:adv} \\
$T(M(S), M(S')) : [0,1] \to [0,1]$ & Trade-off curve between FPR and FNR of optimal attacks & \cref{def:trade-off} \\
$f : [0,1] \to [0,1]$ & A lower bound on the trade-off curve for all neighboring datasets & \cref{def:fdp} \\
$P, Q$ & A dominating pair of distributions for a given mechanism $M(S)$ & \cref{def:dom-pair} \\
$X, Y$ & Privacy loss random variables for a given dominating pair $P, Q$ & \cref{def:plrv} \\
\end{tabular}
}
\end{table}
\section{Attack-Aware Noise Calibration with Black-box DP Accountants}
\label{app:blackbox}

\paragraph{Advantage Calibration}
\cref{stmt:dp-to-adv} implies that $(0, \delta)$-DP mechanisms ensure bounded advantage $\eta \leq \delta$. Therefore, given access to a black-box accountant $\varepsilon_\noiseparams(\delta)$ or $\delta_\noiseparams(\varepsilon)$ we can calibrate to a given level of advantage $\eta\target$ by ensuring $(0, \eta\target)$-DP:
\begin{align}\label{eq:adv-calibration-blackbox}
    \min_{\noiseparams \in \noiseparamspace} \noiseparams \quad \text{ s.t. } \quad \varepsilon_\noiseparams(\eta\target) = 0 \quad \text{ or }\quad \delta_\noiseparams(0) = \eta\target
\end{align}
This is a more generic way to perform advantage calibration using an arbitrary black-box accountant. It is equivalent to our procedure in \cref{sec:adv-calibration} when using \citet{doroshenko2022connect} accountant.

\paragraph{FPR/FNR Calibration with Grid Search}
Given a black-box DP accountant, i.e., a method which computes the privacy profile $\varepsilon_\noiseparams(\delta)$ of a mechanism $\mechanism_\noiseparams(\cdot)$, we can  approximate $f_\noiseparams(\alpha)$ by discretizing the range of $\delta \in [0, 1]$ and solving \cref{eq:profiles-to-f} as:
\begin{equation}
    \label{eq:profiles-to-f-blackbox}
    f_\noiseparams(\fpr) \geq \sup_{\delta \in \{\delta_1, \delta_2, \ldots, \delta_u\}} \max \{ 0, \ 1 - \delta - e^{\varepsilon_\noiseparams(\delta)} \fpr, \  e^{-\varepsilon_\noiseparams(\delta)} \cdot (1 - \delta - \fpr)\},
\end{equation}
where $0 \leq \delta_1 < \delta_2 < \ldots < \delta_u \leq 1$. It is possible to perform an analogous discretization using $\delta_\noiseparams(\varepsilon)$ and \cref{stmt:profiles-to-f}, in which case we have to additionally choose a bounded subspace $\varepsilon \in [\varepsilon_{\min}, \varepsilon_{\max}] \subset \sR$. Equivalent procedures to \cref{eq:profiles-to-f-blackbox} have previously appeared in \citet{nasr2023tight,zheng2020sharp}.

Plugging in \cref{eq:profiles-to-f-blackbox} into the problem in \cref{eq:f-calibration}, we can calibrate mechanisms to a given $\alpha\target, \beta\target$ using binary search (see \cref{sec:intro-to-calibration}) in a space $[\noiseparams_{\min}, \noiseparams_{\max}] \subseteq \noiseparamspace$ to additive error $\noiseparams_{\text{err}} > 0$. Denoting by $\nu$:
\begin{equation}\nu \define \frac{\noiseparams_{\max} - \noiseparams_{\min}}{\noiseparams_{\text{err}}},
\end{equation}
the calibration requires $u \cdot \lceil \log_2 \nu \rceil$ evaluations of $\varepsilon_\noiseparams(\delta)$. For instance, a single evaluation of the bound in \cref{eq:profiles-to-f-blackbox} takes approximately one minute with $u = 100$, and six minutes with $u = \num{1000}$ for DP-SGD with $T = \num{10000}$ using \citet{gopi2021numerical} accountant as an instantiation of $\varepsilon_\noiseparams(\delta)$ on commodity hardware (see \cref{app:exp}).
In contrast, evaluating $f_\noiseparams(\cdot)$ using \cref{alg:get-beta} in the same settings takes approximately $500$ms at the default discretization level $\Delta = 10^{-4}$ (see \cref{app:connect-the-dots}).

Although this approach is substantially less computationally efficient than our direct procedure in \cref{sec:err-rates-calibration}, its strength is that it can be used to calibrate noise in any DP algorithm which provides a way to compute its $(\varepsilon, \delta)$ guarantees.

\section{Detailed Calibration Algorithms}
\label{app:calibration-algos}

\paragraph{Advantage calibration} The standard advantage calibration first finds $\varepsilon\target$ for a given $\delta\target < \nicefrac{1}{n}$ which provides the desired advantage guarantee via \cref{eq:dp-to-adv}, then calibrates noise to the derived $(\varepsilon\target, \delta\target)$-DP guarantee using the privacy profile $\varepsilon_\noiseparams(\delta)$ function:

\begin{algorithm}[H]
\caption{Standard advantage calibration}
\label{alg:standard-adv-calibration}
\begin{algorithmic}[1]
\Require $\eta\target, \delta\target$, where $\delta\target < \frac{1}{n}$, privacy profile $\varepsilon_\noiseparams(\delta)$.
\State Find $\varepsilon\target$ by solving \cref{eq:dp-to-adv} for $\varepsilon$ with fixed $\delta = \delta\target$ and $\eta = \eta\target$
\State Find noise parameter $\noiseparams^*$, e.g., using binary search:
\[\omega^* \leftarrow \argmin_{\noiseparams \in \noiseparamspace} \text{ s.t. } \varepsilon_\noiseparams(\delta\target) \geq \varepsilon\target\]
\State \Return $\omega^*$
\end{algorithmic}
\end{algorithm}

For direct calibration to advantage, we first show how to practically use the expression in \cref{stmt:plrv-to-risks} to evaluate advantage using PLRVs:

\begin{algorithm}[H]
\caption{Compute advantage using PLRVs $(X, Y)$}\label{alg:get-adv}
\begin{algorithmic}[1]
\Require PMF $\Pr[X_\noiseparams = \tau]$ over grid $\{x_1, x_2, \ldots, x_k\}$ with $x_1 < x_2 < \ldots < x_k$
\Require PMF $\Pr[Y_\noiseparams = \tau]$ over grid $\{y_1, y_2, \ldots, y_l\}$ with $y_1 < y_2 < \ldots < y_l$
\Procedure{ComputeAdv}{$\noiseparams; X_\noiseparams, Y_\noiseparams$}
    \State $t_X \leftarrow \min \{ i \in [k] \mid x_i > 0 \}$, $t_Y \leftarrow \min \{ i \in [l] \mid y_i > 0 \}$
    \State \Return $\sum_{i = t_Y}^{l} \Pr[Y_\noiseparams = y_i] - \sum_{i = t_X}^{k} \Pr[X_\noiseparams = x_i]$
\EndProcedure
\end{algorithmic}
\end{algorithm}

Given \cref{alg:get-adv}, direct calibration to advantage amounts to, e.g., binary search:
\begin{algorithm}[H]
\caption{Direct advantage calibration using PLRVs $(X, Y)$}
\label{alg:direct-adv-calibration}
\begin{algorithmic}[1]
\Require $\eta\target$, PLRVs $X_\noiseparams, Y_\noiseparams$ (see \cref{alg:get-adv} for a more detailed specification)
\State Find noise parameter $\noiseparams^*$, e.g., using binary search:
\[
    \noiseparams^* \leftarrow \argmin_{\noiseparams \in \noiseparamspace} \text{ s.t. } \textsc{ComputeAdv}(\noiseparams; X_\noiseparams, Y_\noiseparams) \leq \eta\target
\]
\State \Return $\noiseparams^*$
\end{algorithmic}
\end{algorithm}

\paragraph{FPR/FNR Calibration} The standard approach to FPR/FNR calibration proceeds analogously to advantage calibration. First, the algorithm solves \cref{eq:dp-to-f} to obtain the value of $\varepsilon\target$ which ensures that a mechanism satisfies $f(\alpha\target) = \beta\target$. Then, the algorithm calibrates the noise to the obtained $(\varepsilon\target, \delta\target)$ pair using the privacy profile function $\varepsilon_\noiseparams(\delta)$:

\begin{algorithm}[H]
\caption{Standard FPR/FNR calibration}
\label{alg:standard-err-rates-calibration}
\begin{algorithmic}[1]
\Require $\alpha\target, \beta\target, \delta\target$, where $\delta\target < \frac{1}{n}$, privacy profile $\varepsilon_\noiseparams(\delta)$.
\State Find $\varepsilon\target$ by solving \cref{eq:dp-to-f} for $\varepsilon$ with fixed $\delta = \delta\target$ and $f(\alpha\target) = \beta\target$
\State Find noise parameter $\noiseparams^*$, e.g., using binary search:
\[\omega^* \leftarrow \argmin_{\noiseparams \in \noiseparamspace} \text{ s.t. } \varepsilon_\noiseparams(\delta\target) \geq \varepsilon\target\]
\State \Return $\omega^*$
\end{algorithmic}
\end{algorithm}

Direct calibration to FPR/FNR amounts to, e.g., binary search, using calls to \cref{alg:get-beta}:
\begin{algorithm}[H]
\caption{Direct FPR/FNR calibration using PLRVs $(X, Y)$}
\label{alg:direct-err-rates-calibration}
\begin{algorithmic}[1]
\Require $\alpha\target$, $\beta\target$, PLRVs $X_\noiseparams, Y_\noiseparams$ (see \cref{alg:get-beta} for a more detailed specification)
\State Find noise parameter $\noiseparams^*$, e.g., using binary search:
\[
    \noiseparams^* \leftarrow \argmin_{\noiseparams \in \noiseparamspace} \text{ s.t. } \textsc{ComputeBeta}(\noiseparams; \alpha\target; X_\noiseparams, Y_\noiseparams) \geq \beta\target
\]
\State \Return $\noiseparams^*$
\end{algorithmic}
\end{algorithm}
\section{Calibration to Other Risk Notions}
\label{app:other-attack-models}
Noise calibration for a given FPR/FNR level can be seen as a basic building block to calibrate for other operational measures of risk that are functions of FPR $\fpr$ and FNR $\fnr$.

For instance, \citet{rezaei2021difficulty} propose to measure the risks of membership inference attacks in terms of accuracy $\mathtt{acc}$ and FPR $\fpr$, where:
    $\mathtt{acc}(\fpr, \fnr) \define \nicefrac{1}{2} \cdot \left((1 - \fpr) + (1 - \fnr)\right).$
We can calibrate for a given level of accuracy $\mathsf{acc}\target$ and FPR $\fpr\target$ using the method in \cref{sec:err-rates-calibration} by solving the expression for accuracy for a given $\beta\target$.

\citet{jayaraman2021revisiting} propose to measure positive predictive value, or precision, of attacks:
\begin{align}
    \mathtt{ppv}(\fpr, \fnr) \define \frac{1 - \fnr}{1 - \fnr + \fpr}.
\end{align}
Although precision alone is not sufficient to determine the level of privacy, like with accuracy, we can calibrate for a given level of precision $\mathsf{ppv}\target$ \emph{and FPR} $\fpr\target$ by deriving the corresponding $\fnr\target$.

We provide the exact conversions in \cref{tab:other-risk-measures}. These enable practitioners to use the calibration method in \cref{sec:err-rates-calibration} while reporting technically equivalent but potentially more interpretable measures, e.g., attack accuracy at a given FPR.

\begin{table}[t]
    \centering
    \resizebox{.6\linewidth}{!}{
    \begin{tabular}{lll}
         Attack risk measure & Symbol & Derived $\fnr\target$ \\
         \hline
         Advantage & $\eta\target$ & $1 - \fpr\target - \eta\target$ \\
         Accuracy & $\mathtt{acc}\target$ & $2 (\fpr\target -  \mathtt{acc}\target)$ \\
         Positive predictive value / precision & $\mathtt{ppv}\target$ & $\frac{(\alpha\target - 1)(\mathtt{ppv}\target - 1)}{\mathtt{ppv}\target - 1}$
    \end{tabular}
    }
    \vspace{.5em}
    \caption{Some supported risk measures for calibration with a fixed level of FPR $\fpr\target$, with the derivation of the corresponding level of FNR $\fnr\target$. 
    Given $\fpr\target$ and the derived $\fnr\target$, we can calibrate noise using the procedure in \cref{sec:err-rates-calibration}.}
    \label{tab:other-risk-measures}
\end{table}

Although throughout the paper we have assumed that the hypotheses $H_0$ and $H_1$ both have probability $\nicefrac{1}{2}$, our results and conversions can be easily extended to settings where the hypotheses are not equiprobable, as proposed by \citet{jayaraman2021revisiting}.

\section{Dominating Pairs}
\label{app:connect-the-dots}

\subsection{Constructing Discrete Dominating Pairs and their PLRVs}
We summarize the technique from \citet{doroshenko2022connect} to construct a dominating pair from a composed mechanism $\mechanism(S) = \mechanism^{(1)} \circ \mechanism^{(2)} \circ \cdots \circ \mechanism^{(T)}(S)$. This models the common use case in privacy-preserving ML where a simple mechanism, such as the subsampled Gaussian in DP-SGD, is applied $T$ times. We assume that each sub-mechanism $\mechanism^{(i)}, i \in [T]$, has a known privacy curve $\delta_i(\varepsilon)$. Given an input discretization parameter $\Delta$, a size $k$, and a starting $\varepsilon_1$, \cite{doroshenko2022connect} creates a grid $\{ \varepsilon_1, \varepsilon_1 + \Delta, \ldots, \varepsilon_1 + k \Delta\}$. Then, they compute the privacy curve on this grid $\{ \delta_i(\varepsilon_1), \delta_i(\varepsilon_1 + \Delta), \ldots, \delta_i(\varepsilon_1 + k \Delta)\}$, and append the values of $\delta(-\infty) = 0$ and $\delta(\infty)$. The dominating pair for the $i^{\text{th}}$ mechanism is constructed using \cref{alg:compute_dom_pair}. Note that \cref{alg:compute_dom_pair} is identical to Algorithm 1 in \citet{doroshenko2022connect}, with the notation modified to be consistent with the notation in this paper.

This process is repeated for every mechanism. As long as the discretization parameter $\Delta$ is the same for all $T$ mechanisms, the resulting collection of PLRVs can can be composed via the Fast Fourier Transform. The dominating pair for the composed mechanism $\mechanism$ is simply the distribution of $(X_1 + X_2 \ldots + X_T, Y_1 + Y_2 \ldots + X_T)$. 

We remark that the discretization parameter $\Delta$ is user-defined, and the choice for the size $k$ and starting $\varepsilon$ for each grid is mechanism-specific. For further implementation details, we point the reader to the code \href{https://github.com/google/differential-privacy/blob/7126f6194117c54a1a34b75f7ebd47847f317264/common_docs/Privacy_Loss_Distributions.pdf}{documentation file} and the code itself, which can be found in the \texttt{dp\_accounting} Python library,. In particular, we note that while the PLRVs $X, Y$ have the same support except for atoms at $\pm \infty$, the support of the composed PLRV $X_1 + X_2 \ldots +X_T$ need not be the same as the the support of $Y_1 + Y_2 \ldots + Y_T$. This is because in the convolution part of the implementation of \citet{doroshenko2022connect}, the code discards any tail probabilities smaller than some truncation parameter. This is why we allow for $X$ and $Y$ to have different support in \cref{alg:get-beta}, and why we make no assumptions on the distributions of $(P,Q)$ or of $(X,Y)$ in the proof for \cref{stmt:plrv-to-risks}. 
\begin{algorithm}[t]
\caption{\cite{doroshenko2022connect} Construct a dominating pair}\label{alg:compute_dom_pair}
\begin{algorithmic}[1]
\Require Grid: $\{ -\infty, \varepsilon_1, \ldots, \varepsilon_k, \infty\}$. 
\Require Privacy curve on a grid: $\{ 0, \delta(\varepsilon_1), \ldots, \delta(\varepsilon_k), \delta(\infty)\}$. 
\State $P(\infty) = 0$
\For{$i = k-1, \ldots, 1$}
    \State $P(\varepsilon_i) \leftarrow \frac{\delta(\varepsilon_{i-1}) - \delta(\varepsilon_i)}{\exp(\varepsilon_i) - \exp(\varepsilon_{i-1})} - \frac{\delta(\varepsilon_{i}) - \delta(\varepsilon_{i+1})}{\exp(\varepsilon_{i+1})- \exp(\varepsilon_{i})}$
\EndFor
\State $P(-\infty) \leftarrow 1 - \sum_{j \in [k-1]} P(\varepsilon_j) $
\State $Q(-\infty) \leftarrow 0$
\For{$i = 1, \ldots, k-1$}
    \State $Q(\varepsilon_i) \leftarrow \exp(\varepsilon_i) P(\varepsilon_i)$
\EndFor
\State $Q(\infty) = \delta(\infty)$
\State \Return $(P,Q)$
\end{algorithmic}
\end{algorithm}

\subsection{Some Properties of the Trade-Off Curves of Discrete Dominating Pairs}
In this section, we provide several observations on the trade-off curve of discrete dominating pairs. In particular, these observations hold for the trade-off curve described  \cref{stmt:plrv-to-risks}.

\paragraph{Connecting the Dots} From the proof of \cref{stmt:plrv-to-risks} (see \cref{app:proof-plrv-to-risks}), we know that when the level $\fpr$ happens to equal a point in the reverse CDF of $X$, i.e. when $\fpr = \Pr[X > x_i]$ for some $i$, that the corresponding FNR $T(P,Q)(\fpr)$ is simply the CDF of $Y$ evaluated at the same point, i.e. $T(P,Q)(\fpr) = \Pr[Y \leq x_i]$. Since the reverse CDF of $X$ can take on $k+1$ values, it follows that there are $k+1$ values of $\fpr$ where the trade-off curve is fully characterized by the CDF of the PLRVs. 

Next, we observe a special structure of the trade-off curve on the points outside of these $k + 1$ values. For fixed $\tau$, \cref{eq:felipe_eq3} implies $\fpr^*(\tau, \gamma)$ is increasing linearly in $\gamma$  and \cref{eq:felipe_eq4} implies $\fnr^*(\tau, \gamma)$ is decreasing linearly in $\gamma$. This implies that the trade-off curve ``in between'' the $k+1$ points that correspond to the CDFs of the PLRVs is \emph{a linear interpolation}, where one ``connects the dots''. Hence, the trade-off curve is piece-wise linear, continuous everywhere, and not differentiable at the $k+1$ points where $\fpr$ happens to be on the reverse CDF of $X$.

This observation provides an interesting connection to \citet{doroshenko2022connect}, who showed that ``connecting the dots'' between finite points on the \emph{privacy profile} $\delta(e^\varepsilon)$\footnote{The linear interpolation must be done in $e^\varepsilon$ space, as in this grid the privacy profile $\delta(e^\varepsilon)$ is convex.} yields a valid pessimistic estimate to the privacy profile. Could ``connecting the dots'' in trade-off curve space also yield a valid pessimistic estimate? The answer is clearly no: ``connecting the dots'' on finite samples from a trade-off curve corresponds to an optimistic bound on the trade-off curve. Nevertheless, it is interesting to note that the class of discrete and finitely supported privacy loss random variables simultaneously achieve a pessimistic bound in privacy profile space and an optimistic bound in trade-off curve space. Further exploration of this phenomena, specifically in the context of constructing optimal optimistic privacy estimates, is left as future work. 

\paragraph{Behavior at the Edges} The trade-off curve of discrete dominating $(P, Q)$ in general does not satisfy $T(P,Q)(0) = 1$. Indeed, the point $\fpr = 0$ corresponds to $\tau = x_\text{max}$ and $\gamma = 0$, in which case $T(P,Q)(0) = \Pr[Y \leq x_\text{max}] = 1 - \Pr[Y > x_\text{max}]$. Whether or not this equals $1$ depends on the details of the PLRV $Y$, though we note that in our experiments, $T(P,Q)(0)$ is usually 1 to within a margin of $10^{-10}$. Moreover, we have that $T(P,Q)(\fpr) = 0$ for any $\fpr \in [\Pr[X > -\infty] , 1]$. Indeed, for any $\fpr \in [\Pr[X > -\infty] , 1]$, we have that $\tau = -\infty$, meaning that $\fnr^*(\tau,\gamma) = \Pr[Y \leq -\infty] = 0$ for any choice of $\gamma$.

The observation that $T(P,Q)(0) \neq 1$, that $T(P,Q)$ is piece-wise linear, and that $T(P,Q)(\fpr) = 0$ for any sufficiently large $\fpr$, are all consistent with the findings of \citet{discrete_fdp}, who characterized the trade-off curves of discrete-valued mechanisms.

\section{Omitted Proofs}
\label{app:proofs}

\subsection{Omitted Proofs in Section~\ref{sec:method}}

First, let us define the notion of the convex conjugate that we use in the proofs. For a given function $f: [0, 1] \rightarrow [0, 1]$, its convex conjugate $f^*$ is:
\begin{equation}\label{eq:conv-cojugate}
    f^*(y) = \sup_{0 \leq x \leq 1} y x - f(x),
\end{equation}
Next, we can show the omitted proofs.

\pairstof*
\begin{proof}

The proof follows from taking the convex conjugate of both sides of the following result from \citet{zhu2022optimal}:
\begin{proposition}[Lemma 20 from \citet{zhu2022optimal} restated in our notation] If a mechanism is $(\varepsilon, D_{e^\varepsilon}(P ~\|~ Q))$-DP, then it is $f$-DP for $f$ such that the following holds: $$D_{e^\varepsilon}(P ~\|~ Q) = 1 + f^*(-e^\varepsilon).$$
\end{proposition}

Taking the convex conjugate of the equation above reveals that $f$ follows exactly the structure of the trade-off curve implied by the Neyman-Pearson optimal test, which is exactly $T(P,Q)$. See \cref{app:npl} for more details on the Neyman-Pearson lemma.

\end{proof}

\advpitfalls*
\begin{proof}[Proof]
Let us fix a pair of datasets $S \simeq S'$. Suppose that we have a mechanism $\mechanism: 2^{\sD} \rightarrow \{0, 1, 2, 3\}$ which satisfies $(\varepsilon, \delta)$-DP. Further, assume that for the specific fixed pair $S, S'$ it is defined as follows:
\begin{equation}
    \begin{matrix*}[l]
        P(\train(S) = 0) = 0 & P(\train(S') = 0) = \delta \\
        P(\train(S) = 1) = (1 - \delta) \cdot \frac{e^\epsilon}{e^\epsilon + 1}
        & P(\train(S') = 1) = (1 - \delta) \cdot \frac{1}{e^\epsilon + 1} \\
        P(\train(S) = 2) = (1 - \delta) \cdot \frac{1}{e^\epsilon + 1} & P(\train(S') = 2) = (1 - \delta) \cdot \frac{e^\epsilon}{e^\epsilon + 1} \\
        P(\train(S) = 3) = \delta & P(\train(S') = 3) = 0 \\
    \end{matrix*}
\end{equation}
The defining feature of this mechanism is that its trade-off curve $\tradeoff(\mechanism(S), \mechanism(S'))$ for $S, S'$ exactly matches the $f(\cdot)$ curve for generic $(\varepsilon, \delta)$-DP mechanisms in \cref{eq:dp-to-f}~\cite{kairouz2015composition}. Thus, for this mechanism we can use $f$ and $\tradeoff(\mechanism(S), \mechanism(S'))$ interchangeably. In the rest of the proof, we assume that we are calibrating this mechanism.

We want to derive (1) $f_\mathsf{standard}$ under standard calibration with $\delta\target = \nicefrac{1}{c \cdot n}$ and $\varepsilon\target$ chosen such that we have $\eta \leq \eta\target$, (2) $f_\mathsf{adv}$ under advantage calibration for ensuring $\eta\target$, and find their difference.

For this, we first solve \cref{eq:dp-to-adv} for $\varepsilon$ to derive the corresponding $\varepsilon\target$ that would satisfy the required level of $\eta\target$ under standard calibration with $\delta\target = \frac{1}{c \cdot n}$:
\begin{equation}\label{eq:proof-1}
    \varepsilon\target = \log( \frac{2\delta\target - \eta\target - 1}{\eta\target - 1} )
\end{equation}

As we are interested in the low $\fpr$ regime, let us only consider the following form of the DP trade-off curve from \cref{stmt:profiles-to-f}:
\begin{equation}\label{eq:proof-beta-form}
    f(\alpha) = 1 - \delta - e^{\varepsilon} \fpr.
\end{equation}
It is easy to verify that this this form holds for $0 \leq \fpr \leq \frac{1-\delta}{1 + e^{\varepsilon}}$. In the case of $(\varepsilon\target, \delta\target)$-DP with $\varepsilon\target$ defined by \cref{eq:proof-1}, a simple computation shows that this holds for $0 \leq \alpha \leq \frac{1 - \eta\target}{2}$.

To get $f_\mathsf{standard}$, we plug $(\varepsilon\target, \delta\target)$ into the form in \cref{eq:proof-beta-form}.
Recall that by \cref{eq:dp-to-adv} advantage calibration for generic DP mechanisms is equivalent to calibrating noise to $(0, \eta\target)$-DP. Thus, to get $f_\mathsf{adv}(\alpha)$, we plug into $\varepsilon=0, \delta=\eta\target$ to \cref{eq:proof-beta-form}. 
Subtracting the two, we get:
\begin{align}
     \Delta \fnr = \eta\target -  \delta\target + 2 \fpr \frac{\eta\target - \delta\target}{\eta\target - 1},
\end{align}
from which we get the sought form.
\end{proof}

\fprcorrectness*
\begin{proof}[Proof]
    Observe that \cref{alg:get-beta} computes the intermediate values of $\tau$ and $\gamma$ considered in the four cases of $\alpha$ values in the proof of \cref{stmt:plrv-to-risks} given in \cref{app:proof-plrv-to-risks}, and thus computes the valid trade-off curve $T(P, Q)(\fpr)$ as defined in \cref{eq:plrv-to-f}. By \cref{stmt:pairs-to-f}, $\mechanism_\noiseparams(\cdot)$ satisfies $f$-DP with $f = T(P, Q)$.
\end{proof}

\subsection{Proof of Theorem~\ref{stmt:plrv-to-risks}}
\label{app:proof-plrv-to-risks}

\plrvtorisks*

\cref{eq:plrv-to-adv} is an implication of a result by \citet{gopi2021numerical}, which states:
\begin{align}
    \delta(\varepsilon) = \Pr[Y > \varepsilon] - e^\varepsilon \Pr[X > \varepsilon].
\end{align}
We get \cref{eq:plrv-to-adv} by observing that $(0, \delta)$-DP bounds $\eta \leq \delta$ from \cref{stmt:dp-to-adv}.

In the remainder of the proof,  we show \cref{eq:plrv-to-f} and why choosing the threshold $\tau$ and coin flip probability $\gamma$ in the way specified in the theorem guarantees $T(P,Q)(\fpr) = \beta(\tau, \gamma)$. In \cref{app:notation4thm}, we establish the notation necessary for the remainder of the proof along with all the assumptions made. In \cref{app:npl}, we introduce the Neyman-Pearson lemma and use it to construct \cref{eq:plrv-to-f}. Finally, in \cref{app:thmproof}, we prove the final statement of the theorem. 

\subsubsection{Setup,  Notation, and Assumptions}\label{app:notation4thm}
Let the domain of $(P, Q)$ be  $\mathcal{O}$, which we assume to be countable. We refer to the probability mass function of $P$ as $P(\cdot)$ and similarly for $Q$. We allow for multiple atoms $o$  where $P(o) > 0$ and $Q(o) = 0$, and also multiple atoms $o'$ where $Q(o') > 0$ and $P(o') = 0$. We make no further assumptions on $(P,Q)$. 

Since $(P,Q)$ dominate the mechanism $\mechanism(\cdot)$, we know from \cref{stmt:pairs-to-f} that the hypothesis test:
\begin{align}\label{eq:PQtest}
H_0: o &\sim P, \quad H_1: o \sim Q
\end{align}
is easier (the trade-off curve is less than or equal to) that the standard DP hypothesis test:
\begin{align}
H_0: \theta \sim \mechanism(S), \quad H_1: \theta \sim \mechanism(S')
\end{align}
for all $S \simeq S'$. In \cref{app:npl}, we use the Neyman-Pearson Lemma to tightly characterize the trade-off curve implied by \eqref{eq:PQtest}. The notion of privacy loss random variables (PLRVs) $(X, Y)$, which were defined in \cref{def:plrv} as $Y \define \log \nicefrac{Q(o)}{P(o)}$ with $o \sim Q$, and $X \define \log \nicefrac{Q(o')}{P(o')}$ with $o' \sim P$,  appear naturally and play a central role in the proof. 

As such, we establish more notation on them. Let $\mathcal{T}$ denote the finite values that the PLRVs can take $$\mathcal{T} = \{\nicefrac{\log Q(o)}{P(o)} ~\mid~ \: o \in \mathcal{O},~P(o) > 0, ~Q(o) > 0\}.$$ We let the support of $X$ be 
$$\sX = \begin{cases}
    \{-\infty\} \cup \mathcal{T} \quad &\text{ if } \sup \mathcal{T} \in \mathcal{T} \\
    \{-\infty\} \cup \mathcal{T} \cup \{ \sup \mathcal{T} \} \quad &\text{ otherwise}.
\end{cases}$$ 
and we set $\Pr[X = \sup \mathcal{T}] = 0$ if we manually append $\sup \mathcal{T}$ to $X$. We do this to make the quantile of $X$ well-defined on all countable domains. Moreover, let $x_\text{max} = \sup \sX = \sup \mathcal{T}$. We will often refer to elements in the support of $X$ via $\sX = \{ -\infty, x_1, x_2, \ldots, x_\text{max}\}$.

\subsubsection{Applying the Neyman-Pearson Lemma}\label{app:npl}
According to the Neyman-Pearson Lemma~\cite[see, e.g.,][]{Romano2006, dong2019gaussian}, the most powerful attack at level $\fpr$ for the hypothesis test \eqref{eq:PQtest} is a threshold test $\phi^*: \mathcal{O} \rightarrow [0,1]$ parameterized by two numbers $\tau \in \mathbb{R} \cup \{-\infty, \infty\} , \gamma \in [0,1]$,
\begin{align}
    \phi_{\tau, \gamma}^*(o) = \begin{cases}
        1 \quad \text{if } Q(o) > e^\tau P(o)\\
        \gamma \quad \text{if } Q(o) = e^\tau P(o)\\
        0 \quad \text{if } Q(o) < e^\tau P(o).
    \end{cases}
\end{align}
which we can equivalently write as: 
\begin{align}\label{eq:neyman}
    \phi_{\tau, \gamma}^*(o) = \begin{cases}
        1 \quad \text{if } \log \frac{Q(o)}{P(o)} > \tau \\
        \gamma \quad \text{if } \log \frac{Q(o)}{P(o)} = \tau \\
        0 \quad \text{if } \log \frac{Q(o)}{P(o)} < \tau.
    \end{cases}
\end{align}
This threshold test works by flipping a coin and rejecting the null hypothesis (equivalently, guessing that $o$ came from $Q$) with probability $\phi_{\tau, \gamma}^*(o)$. Here, $\log \frac{Q(o)}{P(o)}$ is the Neyman-Pearson test statistic, and $\tau$ is the threshold for this test statistic. If the test statistic is less (greater) than the threshold, the test always rejects (accepts) the null hypothesis, and if the test statistic equals the threshold, the test flips a coin with probability $\gamma$ to reject the null hypothesis. 

The false positive rate of $\phi_{\tau, \gamma}^*$, which we denote by $\fpr$, is the probability that the null hypothesis is rejected ($\phi_{\tau, \gamma}^* > 0$) when the null hypothesis is true ($o \sim P$), and has the following form:
\begin{align}\label{eq:alpha}
    \fpr^*(\tau, \gamma) &\define \E_{o \sim P}[\phi^*_{\tau, \gamma}(o)]\\
    &= \Pr[X > \tau] + \gamma \Pr[X = \tau]. \label{eq:felipe_eq3}
\end{align}
Similarly, the false negative rate of $\phi_{\tau, \gamma}^*$, which we denote $\fnr$, is the probability that the null hypothesis is accepted ($1 - \phi_{\tau, \gamma}^* > 0$) when the null hypothesis is false ($o \sim Q$), and has the following form:
\begin{align}
    \fnr^*(\tau, \gamma) &\define 1- \E_{\theta \sim Q}[\phi^*_{\tau, \gamma}(\theta)] \\
    &= 1 - \left (\Pr[Y > \tau] + \gamma \Pr[Y = \tau] \right )\\
    &= \Pr[Y \leq \tau] - \gamma \Pr[Y = \tau].\label{eq:felipe_eq4}
\end{align}
We have thus shown the correctness of the construction of \cref{eq:plrv-to-f}. In \cref{app:thmproof}, we prove the final statement in \cref{stmt:plrv-to-risks}.

\subsubsection{Construction of the Trade-Off Curve of a Dominating Pair}\label{app:thmproof}
The goal of this section is to prove the following statement made in \cref{stmt:plrv-to-risks}:

\noindent\fbox{%
    \parbox{\textwidth}{%
        For any level $\fpr \in [0,1]$, choosing $\tau = (1-\alpha)$-quantile of $X$ and $\gamma = \frac{\alpha - \Pr[X > \tau]}{\Pr[X=\tau]}$ guarantees that: $$T(P,Q)(\fpr) = \beta^*(\tau, \gamma).$$
    }%
}

where $T(P,Q)(\fpr)$ outputs the false negative rate of the most powerful attack at level $\fpr$. From \cref{app:npl}, we know that the most powerful attack takes the form $\phi^*_{\tau, \gamma}$ as defined in \cref{eq:neyman}. One should think of the level $\fpr$ as a constraint on the attack $\phi^*_{\tau, \gamma}$. In particular, the constraint $\alpha^*(\tau,\gamma) = \alpha$ (where $\alpha^*$ is the false positive rate of $\phi^*_{\tau, \gamma}$ and is defined in \cref{eq:alpha}) yields a family of possible tests that all achieve the level $\fpr$. If $(P,Q)$ were continuous distributions, the constraint $\alpha^*(\tau,\gamma) = \alpha$ would uniquely determine the optimal test. This does not hold in the discrete case, and hence we must identify the most powerful test within this family.

Below, we list out 4 different regimes for the value of the level $\fpr$, identify the family of possible tests in each regime and the most powerful test, and finally give the false negative rate of the respective most powerful test.

\begin{enumerate}[leftmargin=*]
\item[1] \textbf{Case $\fpr = 1$}:
Recall that $X$ has a finite probability of being $-\infty$, meaning that the only way to have $\alpha^*(\tau, \gamma) = 1$ is to set $\tau = -\infty$ and $\gamma = 1$. The corresponding false negative rate is given by $\fnr^*(-\infty, 1) = \Pr[Y \leq -\infty] - \Pr[Y = -\infty] = 0$. 
\item[2] \textbf{Case $\fpr = 0$}:
If we choose the threshold $\tau = x_\text{max}$ and the coin flip probability $\gamma = 0$, then we have that the false positive rate of this test is:
\begin{align}
    \fpr^*(\tau = x_\text{max}, \gamma = 0) &= \Pr(X > x_\text{max}) + \gamma \Pr[X = x_\text{max}] \\ 
    &= 0.
\end{align}
Moreover, any test with $\tau > x_\text{max}$ has $\fpr^*(\tau, \gamma) = 0$. However, increasing the threshold above $x_\text{max}$ can never decrease $\beta^*$. Moreover, a test with a threshold $\tau < x_\text{max}$ cannot achieve $\fpr = 0$. It follows that choosing $(\tau = x_\text{max}, \gamma = 0)$ yields the most powerful test, which has a false negative rate of $\beta^*(x_\text{max}, 0) = \Pr[Y \leq x_\text{max}]$.

\item[3] \textbf{Case $\fpr = \Pr[X > x_t]$ for some $x_t \in \sX$:}
If we choose the threshold $\tau = x_t$ and coin flip probability $\gamma = 0$, then we have that the false positive rate of this test is
\begin{align}
    \fpr^*(\tau = x_t, \gamma = 0) &= \Pr(X > x_t) + 0 \\ 
    &= \fpr.
\end{align}
Moreover, the test $\phi^*_{x_{t+1}, 1}$ and any test with $\tau \in (x_t, x_{t+1})$ has $\alpha^*(\tau, \gamma) =\alpha$. It is straightforward to see that all these tests are equivalent to outputting 1 if  $\log \frac{Q(o)}{P(o)} > x_t$ and 0 otherwise, making them all equivalent to $\phi^*_{x_t, 0}$. Note that no other test can achieve the level $\alpha$, since decreasing the threshold below $x_t$ or above $x_{t+1}$ makes it impossible to achieve level $\alpha$. For fixed threshold $\tau = x_t \: (x_{t+1})$, only a coin flip probability of $\gamma = 0$ or $\gamma = 1$ achieves level $\alpha$. We conclude that all the tests that achieve level $\fpr$  have a false negative rate of $\beta^* = \Pr[Y \leq x_t]$.

\item[4] \textbf{Otherwise}:
If we choose the threshold 
\begin{align}\label{eq:sandwich}
    \tau = \inf\{x \in \sX ~|~ \alpha \geq \Pr[X > x]\}
\end{align}
and choose the coin flip probability $\gamma$ to exactly satisfy the constraint that $\fpr^*(\tau, \gamma) = \alpha$, i.e.,
\begin{equation}\label{eq:gamma}
    \gamma = \frac{\fpr - \Pr[X > x_t]}{\Pr[X = x_t]},
\end{equation} 
then this test achieves a false positive rate of $\fpr$. It is easy to see that this is the only test that achieves level $\alpha$, and has a false negative rate of $\beta^* = \Pr[Y \leq x_t] - \gamma \Pr[Y = x_t]$.

\end{enumerate}
Note that in all regimes, there is one unique test that achieves a level $\fpr$ and is the most powerful test. However, in some regimes of $\fpr \in [0,1]$, namely regime 3, there are many \emph{different parameterizations} for the same test. In these cases, we are free to choose any parameterization. For each regime, the very first test we list is the parameterization we choose. To summarize, we have the following most powerful tests: 

\begin{enumerate}
    \item When $\alpha = 1$, choose $\tau = -\infty, \gamma = 1$
    \item When $\alpha = 0$, choose $\tau = x_\text{max}, \gamma = 0$
    \item When $\alpha = \Pr[X > x_t]$, choose $\tau = x_t, \gamma = 0$
    \item Else, choose $\tau $ via \cref{eq:sandwich}, and $\gamma = \frac{\fpr - \Pr[X > \tau]}{\Pr[X = \tau]}$.
\end{enumerate}

It is clear from the list above that for distributions with finite support, the most powerful test can be concisely written as:
\begin{align}
    \tau &= \inf \{x \in \sX ~\mid~ \alpha \geq \Pr[X > x]\}\label{eq:quantile} \\
    \gamma &=\frac{\fpr - \Pr[X > \tau]}{\Pr[X = \tau]}\label{eq:gamma1}.
\end{align}
where we recognize  $\tau$ as the $(1-\alpha)$-quantile of $X$. 

Note that for distributions with countably infinite support, \cref{eq:gamma1} does not capture Case 2, since $\Pr[X = x_\text{max}] = 0$. So, we define $\gamma = 0$ whenever $\alpha = 0$, and $\gamma = $ \cref{eq:gamma1} otherwise. Since this work focuses on using PLRVs from \citet{doroshenko2022connect}, which are always finitely supported, we report \cref{eq:quantile} and \cref{eq:gamma1} without this edge case in the main body.

We remark that similar results regarding the trade-off curve between two discrete mechanisms can be found in \citet{discrete_fdp}. We differ from this work by parameterizing the trade-off curve using PLRVs, in contrast to \citeauthor{discrete_fdp}, who parameterized the trade-off curve in terms of the discrete distributions $P$ and $Q$. Our parameterization lends itself more naturally to composition, as the PLRVs sum under composition.

\section{Practical Considerations}
\label{app:practical-considerations}

The algorithm of \citet{doroshenko2022connect}, which is implemented in the \texttt{dp\_accounting} Python library,\footnote{\href{https://github.com/google/differential-privacy/tree/0e99a6ff2af75168a22f82548375102fa1a48c8e/python/dp_accounting/dp_accounting/pld}{https://github.com/google/differential-privacy/tree/main/python/dp\_accounting/dp\_accounting/pld}} handles Poisson subsampling under composition (i.e. accounting for DP-SGD) by analyzing the removal and add relations separately. This approach, to the authors knowledge, was first advocated for by \citet{zhu2022optimal} (see the discussion in their Appendix). 

In particular, instead of the algorithm outputting a dominating pair $(P,Q)$ that dominates for the symmetric add/remove relation under composition, it outputs one dominating pair for the asymmetric remove relation $(P_\text{remove},Q_\text{remove})$ and one for the asymmetric add relation $(P_\text{add},Q_\text{add})$. This means that naively applying \cref{stmt:plrv-to-risks} to, for example, $(P_\text{add},Q_\text{add})$, will return a trade-off curve that is only valid for DP-SGD under the  asymmetric add relation. 

To handle the case when \cref{stmt:plrv-to-risks} is applied to a dominating pair $(P,Q)$ (equivalently, the PLRVs $(X,Y)$) that only dominate a mechanism under an asymmetric neighboring relation, a more sophisticated technique is needed to map $T(P,Q)$ to the target symmetric neighboring relation. In particular, a result from \cite{dong2019gaussian} explains how to handle this case:

\begin{proposition}[Proposition F.2 from \citet{dong2019gaussian}]\label{stmt:f-conjugate} Let $f: [0,1] \rightarrow [0,1]$ be a convex, continuous, non-increasing function with $f(x) \leq 1-x$ for $x \in [0,1]$. Suppose a mechanism $\mechanism$ is $(\varepsilon, 1 + f^*(-e^\varepsilon))$-DP for all $\varepsilon \geq 0$, then it is $\textrm{Symm}(f)$-DP with the symmetrization operator $\textrm{Symm}(f)$ defined as:
\begin{equation}
    \textrm{Symm}(f)(x) = 
    \begin{cases}
        \{f, f^{-1} \}^{**}, \quad &\text{ if } \bar{x} \leq f(\bar{x}), \\
        \max \{f, f^{-1} \}, &\text{ if } \bar{x} > f(\bar{x}),
    \end{cases}
\end{equation}
where $\bar{x} = \inf \{ x \in [0,1] ~|: -1 \in \partial f(x)\}$, and 
\begin{equation}
    \{f, f^{-1} \}^{**}(x) = 
    \begin{cases}
        f(x), \quad &\text{ if } x\leq \bar{x}, \\
        f(\bar{x}), &\text{ if } \bar{x} < x \leq f(\bar{x}), \\ 
        f^{-1}(x), &\text{ if } x > f(\bar{x}).
    \end{cases}
\end{equation}

\end{proposition}
Though not explicitly stated, the proposition does assume the mechanism $\mechanism(\cdot)$ has a symmetric neighboring relation. By letting $f$ be unspecified however, the proposition allows for the input function $f$ to correspond to an asymmetric neighboring relation. In this case, the proposition returns a trade-off curve that holds for the symmetric neighboring relation. 

We can hence apply this proposition to the problem at hand by recalling that given a dominating pair $(P,Q)$, we have that the mechanism is $(\varepsilon, D_{e^\varepsilon}(P ~\|~ Q))$-DP. Moreover, \cref{stmt:plrv-to-risks} outputs the trade-off function $f = T(P,Q)$, which is exactly the function $f$ such that $D_{e^\varepsilon}(P ~\|~ Q) = 1 + f^*(-e^\varepsilon)$. We can thus restate \cref{stmt:f-conjugate} in more familiar form as: 

\begin{proposition}[Proposition F.2 from \citet{dong2019gaussian} restated]\label{stmt:f-conjugate1} Suppose that $(P,Q)$ is a dominating pair for a mechanism $\mechanism(\cdot)$ under either the add or remove relation.  Then, the mechanism is $\textrm{Symm}(T(P,Q))$-DP with respect to the add/remove relation. 
\end{proposition}
\cref{stmt:f-conjugate1} allows us to, for example, use a dominating pair for the asymmetric add relation to obtain a trade-off curve for the  symmetric add/remove relation. Moreover, the operator $\textrm{Symm}(T(P,Q))$ turns out to be straightforward to implement in practice. 

\cref{app:thmproof} details how to explicitly construct $T(P,Q)$. It is well known that $T(Q,P)(\alpha) = T(P,Q)^{-1}(\alpha)$, hence the order of $(P,Q)$ can be easily swapped in \cref{app:thmproof} to get the inverse function $T(P,Q)^{-1}$. The only obstacle remaining is in determining $\bar{x} = \inf \{ x \in [0,1] ~|: -1 \in \partial f(x)\}$. Due to the structure of  $T(P,Q)$, namely that it is a piece-wise linear function parameterized by \cref{eq:felipe_eq3} and \cref{eq:felipe_eq4}, it turns out that the subdifferential $\partial f(x)$ are of the form $\{e^\tau\}$, where $\tau$ are the allowable thresholds of the Neyman-Pearson lemma at level $x$ identified in each of the 4 cases of the proof laid out in \cref{app:thmproof}. As an example, a unique threshold of $-\infty$ at $\alpha = 1$ implies that the derivative of $T(P,Q)$ at $\alpha = 1$ is $0$, meaning the trade-off curve is flat there. 

It follows that the constraint $\bar{x} = \inf \{ x \in [0,1] ~|: -1 \in \partial f(x)\}$ implies that $\bar{x}$ is the smallest level $\fpr$ where the threshold switches signs, i.e. $\bar{x} = \alpha^*( \tau = 0, \gamma = 0) = \Pr[X > 0]$ and $f(\bar{x}) = \beta^*( \tau = 0, \gamma = 0) = \Pr[Y \leq 0 ]$. This gives us all the information needed to implement the Symm operator.

\section{Calibrating Gaussian Mechanism}
\label{app:calibrating-specific-mechanisms}
In the case where the trade-off curve of the mechanism has a closed form, we can solve the calibration problems in \cref{eq:adv-calibration,eq:f-calibration} exactly without resorting to the numerical procedures in \cref{sec:adv-calibration,sec:err-rates-calibration}.

\begin{definition}
For a given non-private algorithm $\alg: 2^\sD \rightarrow \sR^d$, a Gaussian mechanism (GM) is defined as $\mechanism(S) = \alg(S) + \xi$, where $\xi \sim \mathcal{N}(0, \Delta_2 \cdot \sigma^2 \cdot I_d)$ and $\Delta_2 \define \sup_{S \simeq S'} \| \alg(S) - \alg(S')\|_2$ is the \emph{sensitivity} of $\alg(S)$.
\end{definition}
For the Gaussian mechanism, we can exactly compute the relevant adversary's error rates:
\begin{proposition}[\citet{balle2018improving, dong2019gaussian}]
    \label{stmt:gm-props}
    Suppose that $\mechanism_\sigma(S)$ is GM with sensitivity $\Delta_2$ and noise variance $\sigma^2$. Denote by $\mu = \nicefrac{\Delta_2}{\sigma}$ and by $\Phi(t)$ the CDF of the standard Gaussian distribution $\mathcal{N}(0, 1)$. Then,
    \begin{itemize}
        \item The mechanism satisfies $(\varepsilon, \delta)$-DP if the following holds:
        \begin{align}
            \label{eq:gaussian-adv}
            \delta = \Phi\left(\frac{\mu}{2} - \frac{\varepsilon}{\mu} \right) - e^\varepsilon \Phi\left(-\frac{\mu}{2} - \frac{\varepsilon}{\mu} \right)
        \end{align}
        \item It satisfies $f$-DP with:
        \begin{align}
            \label{eq:gaussian-trade-off}
            f(\alpha) = \Phi\left(\Phi^{-1}(1 - \fpr) - \mu\right)
        \end{align}
    \end{itemize}
\end{proposition}

With these closed-form expressions, we can solve the calibration problems exactly:
\begin{corollary}[Advantage calibration for GM]
    \label{stmt:gm-adv-cal}
    For a GM $\mechanism_\sigma(S)$ and target $\eta\target > 0$, choosing $\sigma$ as:
    \begin{equation}
        \sigma = \frac{\Delta_2}{2 \Phi^{-1}\left(\frac{\eta\target + 1}{2}\right)}
    \end{equation}
    ensures that adversary's advantage is upper bounded by $\eta\target$.
\end{corollary}

\begin{proof}[Proof of \cref{stmt:gm-adv-cal}]
    It is sufficient to ensure $(0, \eta\target)$-DP. Plugging in $\varepsilon = 0$ and $\delta = \eta\target$ into \cref{eq:gaussian-adv}, we have:
    \begin{align}
         \eta\target = \Phi\left(\frac{\mu}{2}\right) - \Phi\left(-\frac{\mu}{2}\right) = 2\Phi\left(\frac{\mu}{2}\right) - 1,
    \end{align}
    from which we can derive $\mu = \frac{\Delta_2}{\sigma} = 2 \Phi^{-1}\left(\frac{\eta\target + 1}{2}\right)$
\end{proof}

By solving \cref{eq:gaussian-trade-off} for $\alpha$, we also have an exact expression for calibrating to a given level of $\fpr\target, \fnr\target$:
\begin{corollary}[FPR/FNR calibration for GM]
For a Gaussian mechanism $\mechanism_\sigma(S)$, and target $\fpr\target \geq 0$, $\fnr\target \geq 0$ such that $\alpha\target + \beta\target \leq 1$, choosing $\sigma$ as:
\begin{equation}
    \sigma = \frac{\Delta_2}{\Phi^{-1}(1 - \fpr\target) - \Phi^{-1}(\fnr\target)}
\end{equation}
ensures that adversary's FNR and FPR rates are lower bounded by $\alpha\target$ and $\beta\target$, respectively.
\end{corollary}

Note that using the exact expressions above to calibrate Gaussian mechanism offer only computational advantages compared the method in the main body. In terms of resulting noise scale $\sigma$, the results are the same as with generic PLRV-based calibration up to a numerical approximation error.
\section{Additional Experiments, Details, and Figures}
\label{app:exp}

\subsection{Computing Resources}
We use a commodity machine with AMD Ryzen 5 2600 six-core CPU, 16GB of RAM, and an Nvidia GeForce RTX 4070 GPU with 16GB of VRAM to run our experiments. All experiments with deep learning take up to four hours to finish.

\subsection{Experimental Setup}
In all our experimental results, the neighborhood relation $S \simeq S'$ is the add-remove relation, i.e., $S \simeq S'$ iff $|S~\Delta~S'| = 1$, which is the standard relation used by modern DP-SGD accountants. See more on implementation details related to the neighborhood relation in \cref{app:practical-considerations}.

\paragraph{Text Sentiment Classification} We follow \citet{yu2021differentially} to finetune a GPT-2 (small)~\cite{radford2019language} using LoRA~\cite{hu2021lora} with DP-SGD on the SST-2 sentiment classification task~\cite{socher2013recursive} from the GLUE benchmark~\cite{wang2018glue}. We use the Poisson subsampling probability $p \approx 0.004$ corresponding to expected batch size of 256, gradient clipping norm of $\Delta_2 = 1.0$, and finetune for three epochs with LoRA of dimension 4 and scaling factor of 32. We vary the noise multiplier $\sigma \in \{0.5715, 0.6072, 0.6366, 0.6945, 0.7498\}$ approximately corresponding to $\varepsilon \in \{3.95, 3.2, 2.7, 1.9, 1.45\}$, respectively, at $\delta = 10^{-5}$. We use the default training split of the SST-2 dataset containing \num{67348} examples for finetuning, and the default validation split containing 872 examples as a test set.

\paragraph{Image Classification} We follow \citet{tramer2021differentially} to train a convolutional neural network~\citep[Table 9, Appendix]{tramer2021differentially} over the ScatterNet features~\cite{oyallon2015deep} on the CIFAR-10~\cite{krizhevsky2009learning} image classification dataset. We use the Poisson subsampling probability of $p \approx 0.16$  corresponding to expected batch size of 8192, learning rate of 4, Nesterov momentum of 0.9, and gradient clipping norm of $\Delta_2 = 0.1$. We  train for up to $100$ epochs. We vary the gradient noise multiplier $\nicefrac{\sigma}{\Delta_2} \in \{4, 5, 6, 8, 10\}$, corresponding to $\varepsilon \in \{5, 3.86, 3.15, 2.31, 1.63\}$, respectively, at $\delta = 10^{-5}$. We use the default 50K/10K train/test split of CIFAR-10.

\subsection{Additional Experiments with Histogram Release}

Histogram release is a simple but common usage of DP, appearing as a building block, e.g., in private query interfaces~\cite{gaboardi2020programming}. To evaluate attack-aware noise calibration for histogram release, we use the well-known ADULT dataset~\cite{adult} comprising a small set of US Census data. We simulate the release of the histogram of the `Education' attribute (with 16 distinct values, e.g., ``High school'', ``Bachelor’s'', etc.) using the standard Gaussian mechanism with post-processing to ensure that the counts are positive integers. To measure utility, we use the $L_1$ distance (error) between the original histogram and the released private histogram.

\cref{fig:histograms} shows the increase in utility if we calibrate the noise of the mechanism using the direct calibration algorithm to a given level of FPR $\alpha\target$ and FNR $\beta\target$ vs. standard calibration over 100 simulated releases with different random seeds. In certain cases, e.g., for $\alpha\target = 0.1$ and $\beta\target = 0.75$, our approach decreases the error by approx. $3\times$ from three erroneous counts on average to one.

\subsection{Software} We use the following key open-source software:
\begin{itemize}
    \item PyTorch~\cite{pytorch} for implementing neural networks.
    \item huggingface~\cite{huggingface} suite of packages for training language models.
    \item opacus~\cite{opacus} for training PyTorch neural networks with DP-SGD.
    \item dp-transformers~\cite{dp-transformers} for differentially private finetuning of language models.
    \item numpy~\cite{numpy}, scipy~\cite{scipy}, pandas~\cite{pandas}, and jupyter~\cite{jupyter} for numeric analyses.
    \item seaborn~\cite{seaborn} for visualizations.
\end{itemize}

\clearpage

\ifthenelse{\boolean{preprint}}{}{
\begin{figure}
    \centering
    \resizebox{\figwidth}{!}{
    \begin{tikzpicture}[
        declare function={normcdf(\x) = 1/(1 + exp(-0.07056*((\x)/1)^3 - 1.5976*(\x)/1));
                          _invpart(\p) = -5.531 * (((1 - \p) / \p)^0.1193 - 1);
                          invnormcdf(\p) = (\p >= 0.5) * _invpart(\p) + (\p < 0.5) * (-_invpart(1 -\p));
        }
    ]
        \begin{axis}[
            xlabel={Attack FPR, $\fpr$},
            ylabel={Attack FNR, $\fnr$},
            domain=0:1,
            xmin=0,
            xmax=1,
            ymin=0,
            ymax=1,
            samples=300, %
            legend pos=north east, %
            legend cell align={left},
            axis line style={draw opacity=0.5},
            trim axis left,
            trim axis right,
            enlargelimits=false,
            xtick={0, 0.5, 1},
            ytick={0, 0.5, 1},
            extra y ticks={0.9, 0.165},
            extra y tick labels={$1 - \delta$, $\frac{1 - \delta}{1 + e^\varepsilon}$},
            extra y tick style={
                grid=major,
                yticklabel style={yshift=0.7ex, anchor=east}
            },
            extra x ticks={0.9, 0.165},
            extra x tick labels={$1 - \delta$, $\frac{1 - \delta}{1 + e^\varepsilon}$},
            extra x tick style={
                grid=major,
                yticklabel style={xshift=0.7ex, anchor=north}
            },
        ]
        \def\eps{1.5}
        \def\mu{1.1}
        \def\etacomputed{(exp(\eps) - 1 + 0.2) / (1 + exp(\eps))}
        \addplot[color=gray, domain=0:1, ultra thick, fill, fill opacity=0.2] {max(-10000 * x + 1, 1 - 0.1 - exp(\eps) * x, exp(-\eps) * (1 - 0.1 - x), 0};
        \addlegendentry{Feasible region by DP}
        
        \addplot[color=gray, domain=0:1, ultra thick, dotted, fill, fill opacity=0.2] {normcdf(invnormcdf(1 - x) - \mu)};
        \addlegendentry{Exact feasible region}
        \end{axis}
    \end{tikzpicture}
    }
    \caption{Trade-off curves of a Gaussian mechanism that satisfies $(\varepsilon, \delta)$-DP.
    Each curve shows a boundary of the feasible region (greyed out) of possible membership inference attack FPR ($\fpr$) and FNR ($\fnr$) pairs. The solid curve shows the limit of the feasible region guaranteed by DP via \cref{eq:dp-to-f}, which is a conservative overestimate of attack success rates compared to the exact trade-off curve (dotted). The maximum advantage $\eta$ is achieved with FPR and FNR at the point closest to the origin.}
    \label{fig:dp-curve}
\end{figure}
}

\begin{figure}[t]
    \centering
    \resizebox{\figwidth}{!}{
    \begin{tikzpicture}[
        declare function={normcdf(\x) = 1/(1 + exp(-0.07056*((\x)/1)^3 - 1.5976*(\x)/1));
                          _invpart(\p) = -5.531 * (((1 - \p) / \p)^0.1193 - 1);
                          invnormcdf(\p) = (\p >= 0.5) * _invpart(\p) + (\p < 0.5) * (-_invpart(1 -\p));
        }
    ]

        \begin{axis}[
            xlabel={Attack FPR, $\fpr$},
            ylabel={Attack FNR, $\fnr$},
            domain=0:1,
            xmin=0,
            xmax=1,
            ymin=0,
            ymax=1,
            samples=100, %
            legend pos=north east, %
            legend cell align={left},
            axis line style={draw opacity=0.5},
            trim axis left,
            trim axis right,
            enlargelimits=false,
            font={\sffamily},
            xtick={0, 0.5, 1},
            ytick={0, 0.5, 1}
        ]
        \def\eps{0.73}
        \def\delta{0.05}
        \def\etacomputed{(exp(\eps) - 1 + 2*\delta) / (1 + exp(\eps))}

        \addplot[color=seabornblue, domain=0:1] {max(0, 1 - \delta - exp(\eps) * x, exp(-\eps) * (1 - \delta - x)};
        \addlegendentry{Standard calibration (SC)}

        \def\mucf{0.6}
        \addplot[color=seabornblue, domain=0:1, ultra thick, dotted] {normcdf(invnormcdf(1 - x) - \mucf)};
        \addlegendentry{Exact SC trade-off curve}

        \addplot[color=seabornorange, domain=0:1 - \etacomputed] {1 - \etacomputed - x)};
        \addlegendentry{Advantage calibration (AC)}

        \def\mucal{0.9999}
        \addplot[color=seabornorange, domain=0:1, ultra thick, dotted] {normcdf(invnormcdf(1 - x) - \mucal)};
        \addlegendentry{Exact AC trade-off curve}

        \end{axis}
    \end{tikzpicture}
    }
    \caption{The increase in attack sensitivity due to calibration for advantage is less drastic for Gaussian mechanism than for a generic $(\varepsilon, \delta)$-DP mechanism. 
    }
    \label{fig:adv-pitfalls-exact}
\end{figure}

\begin{figure}
    \centering
    \resizebox{\linewidth}{!}{
    \input{images/histogram_err_rates_calibration.pgf}
    }
    \caption{Direct calibration to attack FNR/FPR reduces average $L_1$ error in histogram release with Gaussian mechanism. The confidence bands are $95\%$ CI over 100 simulated releases.}
    \label{fig:histograms}
\end{figure}

\clearpage

\newpage
\section*{NeurIPS Paper Checklist}

\begin{enumerate}

\item {\bf Claims}
    \item[] Question: Do the main claims made in the abstract and introduction accurately reflect the paper's contributions and scope?
    \item[] Answer: \answerYes{} %
    \item[] Justification: Claims in the abstract/intro succinctly represent the claims in the main body.

\item {\bf Limitations}
    \item[] Question: Does the paper discuss the limitations of the work performed by the authors?
    \item[] Answer: \answerYes{} %
    \item[] Justification: \cref{sec:adv-calibration} discusses in detail the limitations of advantage calibration. 
    \cref{sec:conclusions} discusses limitations and future work for the whole paper.
    \item[] Guidelines:
    \begin{itemize}
        \item The answer NA means that the paper has no limitation while the answer No means that the paper has limitations, but those are not discussed in the paper. 
        \item The authors are encouraged to create a separate ``Limitations" section in their paper.
        \item The paper should point out any strong assumptions and how robust the results are to violations of these assumptions (e.g., independence assumptions, noiseless settings, model well-specification, asymptotic approximations only holding locally). The authors should reflect on how these assumptions might be violated in practice and what the implications would be.
        \item The authors should reflect on the scope of the claims made, e.g., if the approach was only tested on a few datasets or with a few runs. In general, empirical results often depend on implicit assumptions, which should be articulated.
        \item The authors should reflect on the factors that influence the performance of the approach. For example, a facial recognition algorithm may perform poorly when image resolution is low or images are taken in low lighting. Or a speech-to-text system might not be used reliably to provide closed captions for online lectures because it fails to handle technical jargon.
        \item The authors should discuss the computational efficiency of the proposed algorithms and how they scale with dataset size.
        \item If applicable, the authors should discuss possible limitations of their approach to address problems of privacy and fairness.
        \item While the authors might fear that complete honesty about limitations might be used by reviewers as grounds for rejection, a worse outcome might be that reviewers discover limitations that aren't acknowledged in the paper. The authors should use their best judgment and recognize that individual actions in favor of transparency play an important role in developing norms that preserve the integrity of the community. Reviewers will be specifically instructed to not penalize honesty concerning limitations.
    \end{itemize}

\item {\bf Theory Assumptions and Proofs}
    \item[] Question: For each theoretical result, does the paper provide the full set of assumptions and a complete (and correct) proof?
    \item[] Answer: \answerYes{} %
    \item[] Justification: The theoretical results are within the standard setup of differential privacy detailed in \cref{sec:background-dp}.

    \item {\bf Experimental Result Reproducibility}
    \item[] Question: Does the paper fully disclose all the information needed to reproduce the main experimental results of the paper to the extent that it affects the main claims and/or conclusions of the paper (regardless of whether the code and data are provided or not)?
    \item[] Answer: \answerYes{} %
    \item[] Justification: We provide the detailed information on reproducing the experimental results in \cref{app:exp}. Moreover, we link the code along with the instructions for reproducing.

\item {\bf Open access to data and code}
    \item[] Question: Does the paper provide open access to the data and code, with sufficient instructions to faithfully reproduce the main experimental results, as described in supplemental material?
    \item[] Answer: \answerYes{} %
    \item[] Justification: We use common openly available benchmark datasets. We have published the code on the Github platform.

\item {\bf Experimental Setting/Details}
    \item[] Question: Does the paper specify all the training and test details (e.g., data splits, hyperparameters, how they were chosen, type of optimizer, etc.) necessary to understand the results?
    \item[] Answer: \answerYes{} %
    \item[] Justification: We provide the information on the machine learning details in the main body as well as in \cref{app:exp}.

\item {\bf Experiment Statistical Significance}
    \item[] Question: Does the paper report error bars suitably and correctly defined or other appropriate information about the statistical significance of the experiments?
    \item[] Answer: \answerNo{} %
    \item[] Justification: In our setting, we can directly approximate the theoretical quantities of interest (i.e., the level of privacy) without the need for empirical statistical methods and the corresponding uncertainty estimation. For the empirically evaluated model accuracy values, we only use one seed in the main suite of experiments for computational reasons. In the additional experiments in \cref{app:exp}, we provide $95\%$ confidence bands.
    \item[] Guidelines:
    \begin{itemize}
        \item The answer NA means that the paper does not include experiments.
        \item The authors should answer ``Yes" if the results are accompanied by error bars, confidence intervals, or statistical significance tests, at least for the experiments that support the main claims of the paper.
        \item The factors of variability that the error bars are capturing should be clearly stated (for example, train/test split, initialization, random drawing of some parameter, or overall run with given experimental conditions).
        \item The method for calculating the error bars should be explained (closed form formula, call to a library function, bootstrap, etc.)
        \item The assumptions made should be given (e.g., Normally distributed errors).
        \item It should be clear whether the error bar is the standard deviation or the standard error of the mean.
        \item It is OK to report 1-sigma error bars, but one should state it. The authors should preferably report a 2-sigma error bar than state that they have a 96\% CI, if the hypothesis of Normality of errors is not verified.
        \item For asymmetric distributions, the authors should be careful not to show in tables or figures symmetric error bars that would yield results that are out of range (e.g. negative error rates).
        \item If error bars are reported in tables or plots, The authors should explain in the text how they were calculated and reference the corresponding figures or tables in the text.
    \end{itemize}

\item {\bf Experiments Compute Resources}
    \item[] Question: For each experiment, does the paper provide sufficient information on the computer resources (type of compute workers, memory, time of execution) needed to reproduce the experiments?
    \item[] Answer: \answerYes{} %
    \item[] Justification: Our experiments only require commodity hardware. We detail the requirements in \cref{app:exp}.
    
\item {\bf Code Of Ethics}
    \item[] Question: Does the research conducted in the paper conform, in every respect, with the NeurIPS Code of Ethics \url{https://neurips.cc/public/EthicsGuidelines}?
    \item[] Answer: \answerYes{} %
    \item[] Justification: Neither the research process itself nor the outcomes of the research carry significant potential for harm.

\item {\bf Broader Impacts}
    \item[] Question: Does the paper discuss both potential positive societal impacts and negative societal impacts of the work performed?
    \item[] Answer: \answerYes{} %
    \item[] Justification: The topic of our paper is concerned with a social issue of privacy in machine learning and statistical analyses, and our work aims to improve the state of the art in the area. Although our work is mostly technical, we take a broader look in \cref{sec:intro,sec:conclusions}.
    
\item {\bf Safeguards}
    \item[] Question: Does the paper describe safeguards that have been put in place for responsible release of data or models that have a high risk for misuse (e.g., pretrained language models, image generators, or scraped datasets)?
    \item[] Answer: \answerNA{} %
    \item[] Justification: \answerNA{}

\item {\bf Licenses for existing assets}
    \item[] Question: Are the creators or original owners of assets (e.g., code, data, models), used in the paper, properly credited and are the license and terms of use explicitly mentioned and properly respected?
    \item[] Answer: \answerYes{} %
    \item[] Justification: We cite the dataset sources as well the sources for the key pieces of software used for the experimental evaluations and analyses in the main body and \cref{app:exp}.

\item {\bf New Assets}
    \item[] Question: Are new assets introduced in the paper well documented and is the documentation provided alongside the assets?
    \item[] Answer: \answerNA{} %
    \item[] Justification: \answerNA{}

\item {\bf Crowdsourcing and Research with Human Subjects}
    \item[] Question: For crowdsourcing experiments and research with human subjects, does the paper include the full text of instructions given to participants and screenshots, if applicable, as well as details about compensation (if any)? 
    \item[] Answer: \answerNA{} %
    \item[] Justification: \answerNA{}

\item {\bf Institutional Review Board (IRB) Approvals or Equivalent for Research with Human Subjects}
    \item[] Question: Does the paper describe potential risks incurred by study participants, whether such risks were disclosed to the subjects, and whether Institutional Review Board (IRB) approvals (or an equivalent approval/review based on the requirements of your country or institution) were obtained?
    \item[] Answer: \answerNA{} %
    \item[] Justification: \answerNA{}

\end{enumerate}

\end{document}